\def\ARXIVBUILD{1}
\documentclass[11pt]{article}

\usepackage[
  letterpaper,
  left=0.92in,
  right=0.92in,
  top=0.4in,
  bottom=1.15in,
  headsep=0.16in,
  footskip=0.38in
]{geometry}
\usepackage{times}
\usepackage[T1]{fontenc}
\usepackage{microtype}
\usepackage{graphicx}
\usepackage{amsmath,amssymb,amsthm,booktabs}
\usepackage{multirow}
\usepackage[table]{xcolor}
\usepackage{xspace}
\usepackage{pifont}
\usepackage{float}
\usepackage{subcaption}
\usepackage{wrapfig}
\usepackage{longtable}
\usepackage{needspace}
\usepackage[authoryear,round]{natbib}
\usepackage{titlesec}
\usepackage[most]{tcolorbox}
\usepackage{fancyhdr}
\usepackage{url}
\usepackage{hyperref}

\usepackage{amsmath,amsfonts,bm}

\def\eqref#1{equation~\ref{#1}}

\def\1{\bm{1}}

\DeclareMathAlphabet{\mathsfit}{\encodingdefault}{\sfdefault}{m}{sl}
\SetMathAlphabet{\mathsfit}{bold}{\encodingdefault}{\sfdefault}{bx}{n}

\hypersetup{
  colorlinks=true,
  linkcolor=blue!65!black,
  citecolor=blue!65!black,
  urlcolor=blue!70!black,
  pdfauthor={Yifan Wang et al.},
  pdftitle={Certified Long-Horizon Code Agent Evolution via Validation-Gated Skill Optimization}
}

\definecolor{abstractbg}{RGB}{235,247,252}
\newcommand{\finishdocument}{}

\titleformat{\section}{\Large\bfseries}{\thesection}{0.7em}{}
\titleformat{\subsection}{\large\bfseries}{\thesubsection}{0.7em}{}
\titleformat{\subsubsection}{\normalsize\bfseries}{\thesubsubsection}{0.7em}{}
\titlespacing*{\section}{0pt}{2.2ex plus 0.4ex}{0.8ex}
\titlespacing*{\subsection}{0pt}{1.8ex plus 0.3ex}{0.5ex}

\newcommand{\microsoftwordmark}{%
  \includegraphics[width=1.5in]{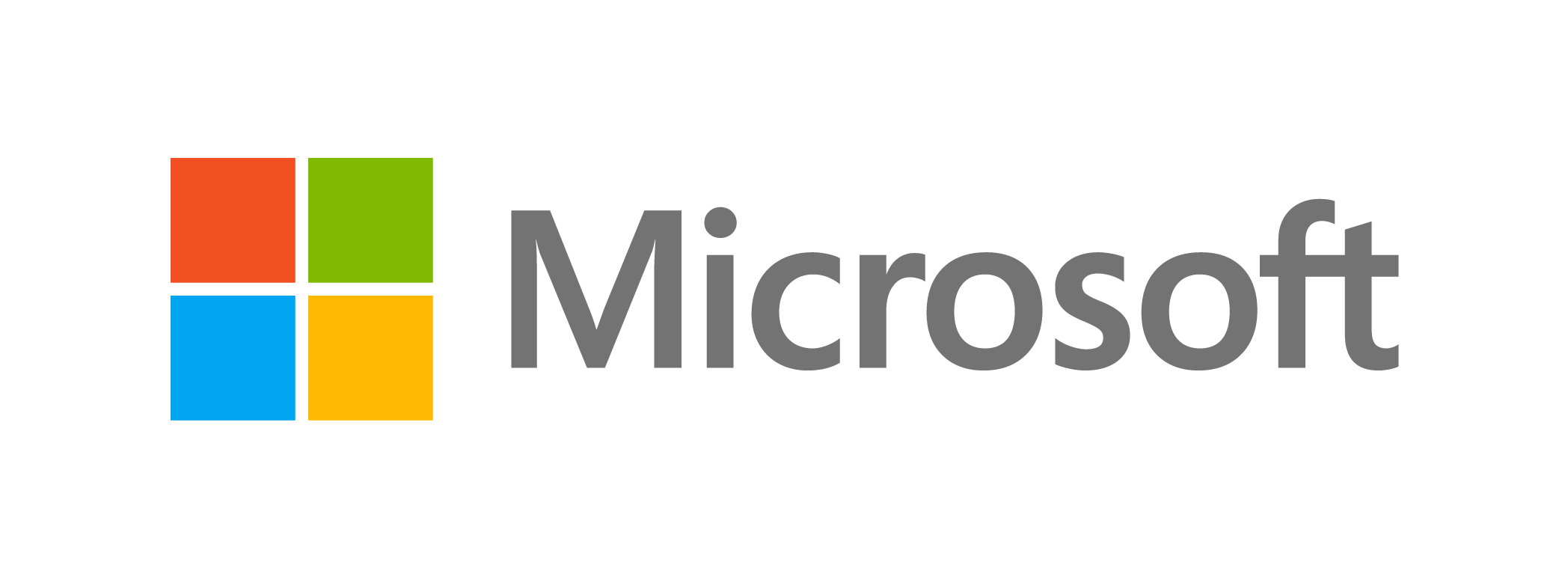}}
\newcommand{\positionedmicrosoftwordmark}{%
  \begin{tikzpicture}[overlay,baseline]
    \node[anchor=base west,inner sep=0pt,yshift=-44pt] at (0,0)
      {\microsoftwordmark};
  \end{tikzpicture}}

\fancypagestyle{preprintfirst}{%
  \fancyhf{}
  \fancyhead[C]{%
    \makebox[\headwidth][s]{%
      \positionedmicrosoftwordmark
      \hfill
      \raisebox{-30pt}[0pt][0pt]{2026 Aug}}}
  \fancyfoot[C]{\thepage}
  \renewcommand{\headrulewidth}{0.7pt}
  \renewcommand{\headrule}{\vskip 26pt\hbox to\headwidth{\color{black!65}\leaders\hrule height \headrulewidth\hfill}}
  \renewcommand{\footrulewidth}{0pt}
}
\fancypagestyle{preprint}{%
  \fancyhf{}
  \fancyfoot[C]{\thepage}
  \renewcommand{\headrulewidth}{0pt}
  \renewcommand{\footrulewidth}{0pt}
}
\makeatletter
\renewcommand{\maketitle}{%
  \thispagestyle{preprintfirst}
  \vspace*{8.35pt}
  {\centering
    {\fontsize{19}{23}\selectfont\bfseries \@title\par}
    \vspace{0.14in}
    {\normalsize \@author\par}}  
  \vspace{0.01in}
}
\makeatother

\renewenvironment{abstract}{%
  \begin{tcolorbox}[
    enhanced,
    colback=abstractbg,
    colframe=abstractbg,
    boxrule=0pt,
    arc=9pt,
    left=17pt,right=17pt,top=11pt,bottom=10pt]
}{%
  \vspace{6pt}
  {\color{black!20}\hrule height 0.5pt}
  \vspace{5pt}
  {\footnotesize
  Work done during an internship at Microsoft Research. Correspondence:
  Yifan Wang (\href{mailto:wang5617@purdue.edu}{wang5617@purdue.edu}).}
  \end{tcolorbox}
}

\begin{document}

\newtheorem{remark}{Remark}
\newtheorem{proposition}{Proposition}
\newtheorem{definition}{Definition}
\newtheorem{theorem}{Theorem}
\newtheorem{lemma}{Lemma}
\newtheorem{corollary}{Corollary}
\theoremstyle{definition}
\newtheorem{assumption}{Assumption}

\definecolor{locushl}{RGB}{219,234,254}
\definecolor{gaingreen}{RGB}{22,163,74}
\definecolor{circlednum}{RGB}{0,0,0}
\definecolor{settinghl}{RGB}{255,244,214}
\definecolor{gainpos}{RGB}{45,125,85}
\definecolor{gainneg}{RGB}{170,70,70}
\newcommand{\circlednum}[1]{\textcolor{black}{\ding{\numexpr171+#1\relax}}}
\providecommand{\gain}[1]{\,\textcolor{gaingreen}{(+#1)}}
\providecommand{\loss}[1]{\,\textcolor{red}{(-#1)}}

\newcommand{\method}{\textsc{\texttt{VALVE}}\xspace}
\newcommand{\ours}{\method}

\newcommand{\fix}{\marginpar{FIX}}
\newcommand{\new}{\marginpar{NEW}}

\makeatletter
\newcommand{\arxivonly}{\ifdefined\ARXIVBUILD\expandafter\@firstofone\else\expandafter\@gobble\fi}
\newcommand{\iclronly}{\ifdefined\ARXIVBUILD\expandafter\@gobble\else\expandafter\@firstofone\fi}
\makeatother

\makeatletter
\newcommand{\markmainend}{\label{sec:mainend}\pdfsavepos\write\@auxout{\string\gdef\string\mainendpos{\the\pdflastypos}\string\typeout{MAINEND ypos=\the\pdflastypos sp page=\thepage}}}
\makeatother

\title{Certified Long-Horizon Code Agent Evolution via Validation-Gated Skill Optimization}
\author{
Yifan Wang$^{1}$ \quad Hao Cheng$^{6}$ \quad Xiaomin Li$^{6}$ \quad
Yuexing Hao$^{6}$ \quad Hemanth Neelgund Ramesh$^{2}$\\
Dongwon Jung$^{3}$ \quad Hao Tang$^{6}$ \quad Keru Wang$^{4}$ \quad
Chenliang Zhou$^{5}$ \quad Qianhui Wu$^{6}$ \quad Wenlin Yao$^{6}$\\
Ananth Grama$^{1}$ \quad Andrzej Banburski-Fahey$^{6}$ \quad
Baolin Peng$^{6}$ \quad Jaron Lanier$^{6}$ \quad Jianfeng Gao$^{6}$\\[4pt]
\normalfont
{\color{black!58}
$^{1}$Purdue University \quad
$^{2}$University of Washington \quad
$^{3}$University of California, Davis\\
$^{4}$New York University \quad
$^{5}$University of Cambridge \quad
$^{6}$Microsoft}
}

\maketitle

\begin{abstract}
Long horizon agent self-evolution without model weight updates is essential for enabling deployed agents to accumulate reusable skills and improve over time.
Prior self-evolution work has focused primarily on short-horizon tasks, while repository-level software engineering remains unexplored despite being an ideal testbed for long-horizon adaptation. In this setting, agents are required to solve streams of sequential tasks, navigate complex dependencies with evolving repositories and persistently store and reuse experience.
Text-based skill optimization offers an efficient,
non-parametric approach for such adaptation. However, existing methods often
suffer from unstable updates, performance drawdown, and agent collapse over
extended deployments. 
In this paper, we formalize the concept of in-context self-evolution
and introduce \ours{}, a validated-gated framework for long-horizon skill optimization. We establish finite convergence, provide theoretical guarantees for
future-task gain and drawdown, and derive the validation and evaluation
holdout sizes required for a prescribed tolerance, with leading-order scaling
$N\gtrsim s^2L/\gamma^2$.

Empirically, our pipeline, \ours{} achieves stable self-improvement over evolution horizon spanning more than $\textbf{1{,}000}$ \textbf{SWE tasks}, with average final and peak gains of $14.9$ and $16.5$ points across three frontier models (GPT-5.5, Claude-4.6 and MiniMax-M2.7). The validation gate reduces average drawdown by $\textbf{75\%}$ and produces an $\mathbf{11\times}$ more compact skill bank than ungated evolution. We further present extensive ablations identifying the design choices most critical to long-horizon skill evolution.
\end{abstract}

\providecommand{\teaserwidth}{\linewidth}
\begin{figure}[H]
\vspace{-1\baselineskip}
\centering
\begin{minipage}{\teaserwidth}
\begin{subfigure}[t]{0.53\linewidth}
    \centering
    \includegraphics[width=\linewidth]{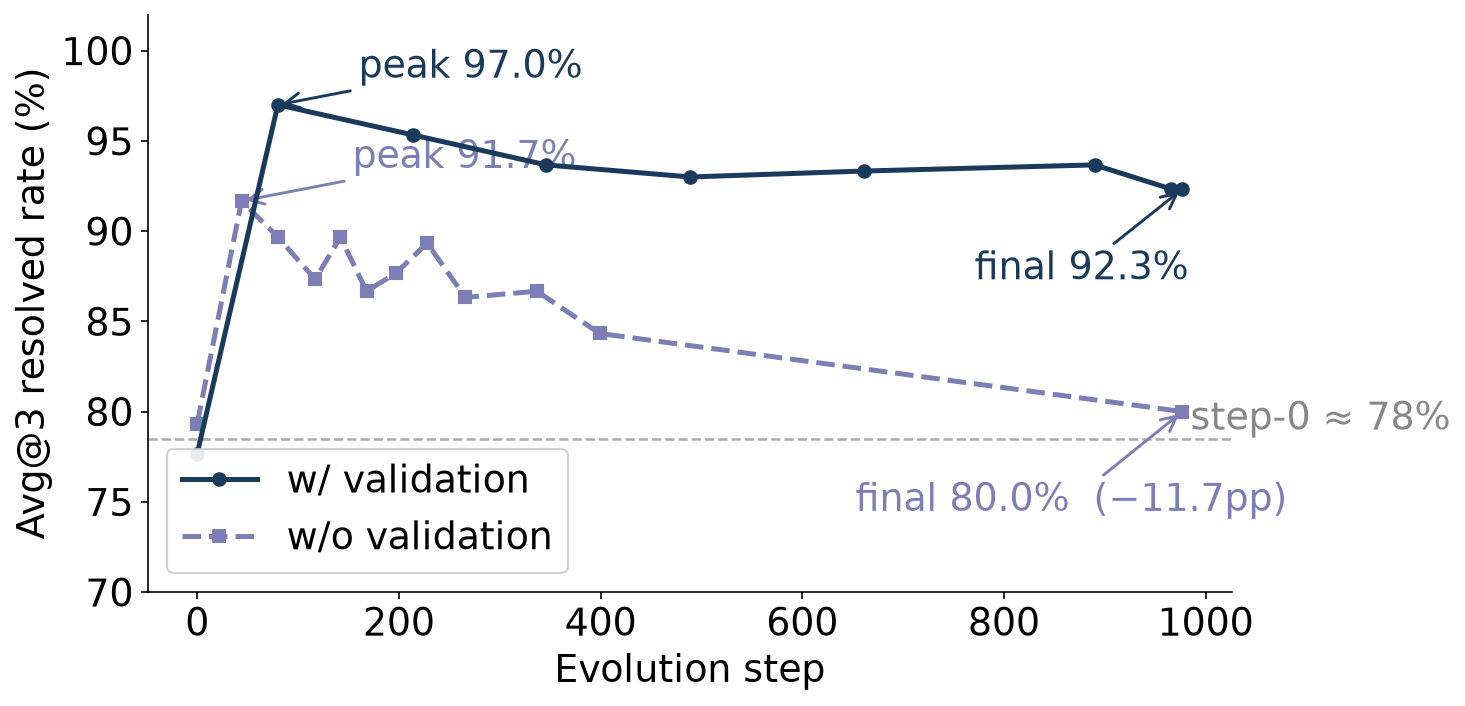}
    \caption{Claude-4.6 evolution trajectory on \texttt{pygments}.}
    \label{fig:evolution-curve-main}
\end{subfigure}
\hfill
\begin{subfigure}[t]{0.46\linewidth}
    \centering
    \includegraphics[width=\linewidth]{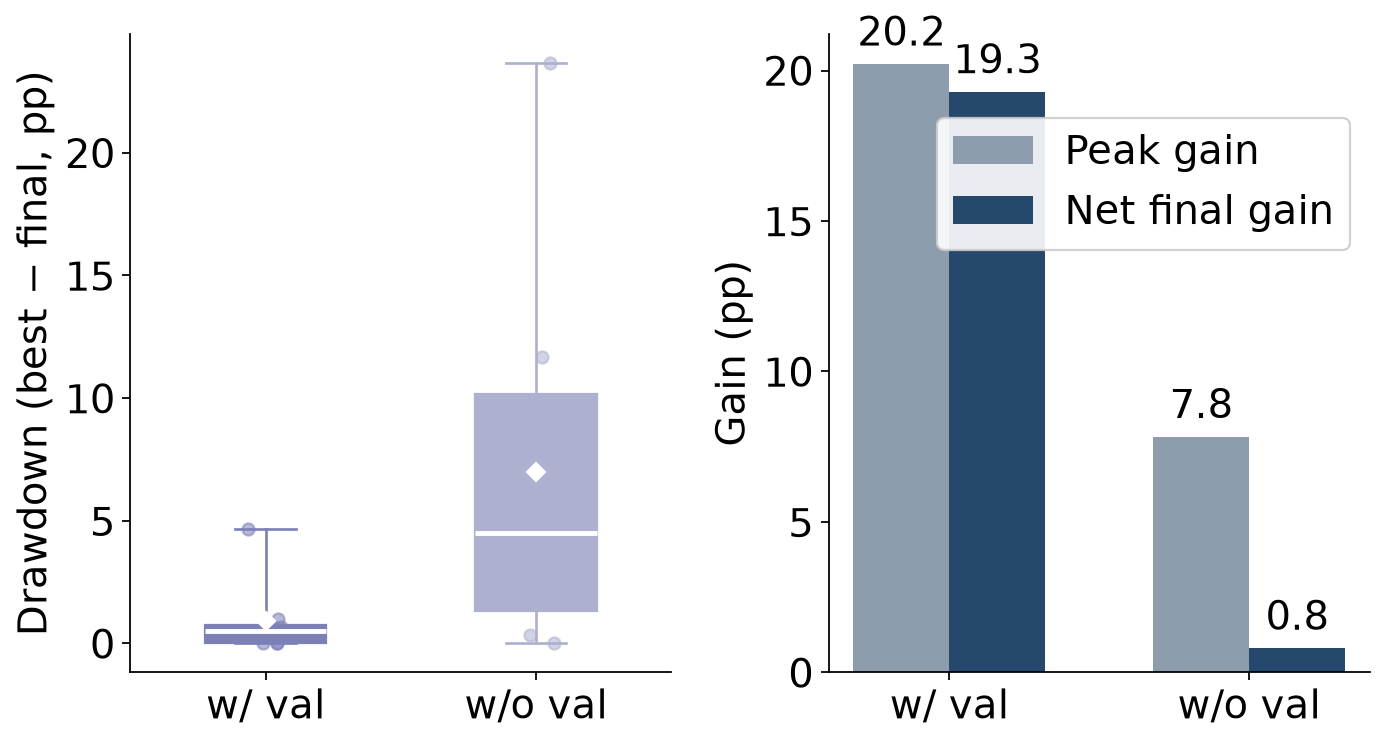}
    \caption{Aggregate drawdown and retained gain.}
    \label{fig:drawdown-main}
\end{subfigure}
\end{minipage}
\vspace{-0.2\baselineskip}
\caption{\textbf{Validation stabilizes long-horizon self-evolution.}
\textbf{(a)} Avg@3 resolved rate over approximately $1{,}000$
SWE-smith task arrivals. Both runs improve early, but the ungated run falls
from a $91.7\%$ peak to $80.0\%$, while the validated run peaks at $97.0\%$ and
ends at $92.3\%$. \textbf{(b)} Across all twelve model--repository runs,
validation reduces average drawdown by $75\%$.}
\label{fig:validation-effect}
\vspace{-1.0\baselineskip}
\end{figure}

\section{Introduction}
\label{sec:intro}

Long-horizon in-context self-evolution requires an agent to continually learning throughout
deployment while its model weights remain frozen.
Repository-level software engineering is an ideal testbed: a code agent must process approximately sequential task arrivals (e.g., up to $1{,}000$ requests), navigate complex codebases and dependencies, and carry experience forward across diverse issues.
After solving a task, it may extract reusable principles and procedures into a persistent skill bank and recall relevant items for future tasks.
This form of test-time adaptation is attractive because it requires neither model retraining nor changes to the solver harness.
It connects prompt optimization \citep{zhou2023ape,pryzant2023protegi,yang2024opro,yuksekgonul2024textgrad,agrawal2025gepa}, context engineering \citep{zhang2025ace,ye2026mce}, skill optimization \citep{yang2026skillopt,shen2026skilloptlite}, and agent self-improvement \citep{zhang2025dgm,wang2025hgm}, all of which adapt external text through execution feedback.
However, most existing methods are evaluated in bounded optimization settings rather than over extended task streams. It remains unclear whether repeated textual updates continue to help when the same persistent memory must support hundreds of subsequent decisions \citep{shridhar2021alfworld,dunn2017searchqa,zhang2025ace,ye2026mce}.

Motivated by recent methods \citep{yang2026skillopt,shen2026skilloptlite}, we study long-horizon self-evolution through a minimal \emph{Extract-Validate-Retrieve} formulation.
Specifically, \emph{Extract} proposes text updates based on completed task trajectories;
\emph{Validate} decides which proposals are retrained in the memory bank using held-out tasks;
and \emph{Retrieve} selects the relevant information from memory for future rollouts.
This formulation exposes the central long-horizon risk.
Without selective retention, individually plausible updates can accumulate, consume context and misguide future task-sovling.
As shown in \autoref{fig:drawdown-main}, an ungated agent may improve early yet lose much of its gain as its bank grows, sometimes finishing below its unevolved counterpart.

This formulation leads directly to a text-space optimization problem.
With the LLM policy and solver harness frozen, the objective is to find a persisent memory bank that maximizes expected solve probability on the repository-level task distribution.
Moreover, we further derive theoretical analyses for the corresponding optimization.
First, the positive-margin gate limits the number of accepted bank changes.
Under i.i.d. sampling and stated variance bounds, a disjoint evaluation set provides high-probability certificates for final population gain and peak-to-final drawdown.
Because the validation set is reused and later proposals may adapt to earlier verdicts, we additionally derive a uniform population-loss bound for every accepted update.
Second, we invert the same deviation bounds to determine the required holdout sizes.
For a prescribed tolerance, the valdiation-panel size controls the certificated population loss of accepted updates.

We instantiate this formation in \method{} 
(\textbf{VAL}idation-gated e\textbf{V}olution for ag\textbf{E}nts) for skill optimization.
On each task, the solver LLM samples a group of attempts.
When their outcomes are mixed, the same LLM then contrasts successful and failed trajectories to propose a set of skills.
The corresponding skill set is only accepted to the memory bank if its paired rollout on a fixed validation achieves desired improvement.
An embedding-based retriever is used to independently ranks validated skills for future task solving.
Thus, all tasks arrivals belong to the delopyment horizon, but only validation-approved skills are recorded in the persistent memory.
To validate the effectiveness of our proposed method and theory, we carry out experiments on four SWE-smith repositories using three frontier LLMs. 
Compared with baselines with and without validation, \method{} achieves the best final gain with minimal drawdown after $1000$ tasks arrivals per repository.


Our contributions are:
\begin{itemize}
    \item \textbf{Long-horizon skill optimization for repo-level software engineering.} We
    introduce a pipeline that combines mixed-outcome contrastive extraction,
    atomic validation of joint proposals with paired executable rollouts, and
    selective retrieval of rich-text principles and skills. Over streams of
    approximately $1{,}000$ task arrivals per repository, \method{} achieves both best 
    average final gain and peak gain across three frontier models and four repositories.

   \item \textbf{Theory for validation-gated evolution.}
    We establish finite convergence, provide guarantees on future-task gain and
    long-horizon drawdown, and derive validation and evaluation set size bounds for
    target performance tolerances.

    \item \textbf{Evidence for selective retention.} Validation reduces
    average drawdown from $6.1$ to $1.5$ points while producing an
    $11\times$ smaller memory bank. Component ablations and item-level analysis show
    how extraction, bank structure, and skill representation affect retained
    improvement.
\end{itemize}

\section{Related Work}
\label{sec:related}

\paragraph{Context, experience, and skill optimization.}
A broad line of work studies how agents can improve by optimizing or
accumulating external textual state while keeping model parameters fixed.
Prompt-optimization methods such as APE, ProTeGi, OPRO, EvoPrompt, TextGrad,
and GEPA iteratively propose and evaluate textual updates
\citep{zhou2023ape,pryzant2023protegi,yang2024opro,guo2024evoprompt,
yuksekgonul2024textgrad,agrawal2025gepa}. Context- and experience-based methods
extend this idea beyond a single prompt: ACE and MCE evolve richer context
artifacts \citep{zhang2025ace,ye2026mce}, while Reflexion, ExpeL, AWM,
ReasoningBank, and Voyager distill past trajectories into reusable feedback,
strategies, workflows, or executable knowledge
\citep{shinn2023reflexion,zhao2024expel,wang2024awm,
ouyang2025reasoningbank,wang2023voyager}. Recent work further formalizes and
constructs agent skills from trajectory evidence
\citep{jiang2026agenticskills,li2026skillsbench,ni2026trace2skill,
liu2026skillforge}. Most closely related, SkillOpt treats skill text as
trainable external state and accepts edits using held-out performance
\citep{yang2026skillopt}, while SkillOpt-Lite introduces an independent
validation gate for skill updates \citep{shen2026skilloptlite}. We study the
same external-state optimization paradigm but focus on a retrieved bank of
individual skills under repeated long-horizon updates, including which updates
to retain and whether evolution remains stable over time.

\paragraph{Self-evolving code agents.}
A separate line of work studies self-improvement specifically in software
engineering. SWE-Exp learns experience banks from repair trajectories
\citep{chen2025sweexp}, and Live-SWE-agent adapts its scaffold online
\citep{xia2025livesweagent}. SICA, the Darwin G\"odel Machine, the
Huxley--G\"odel Machine, and ARTEMIS optimize agent implementations, harnesses,
or configurations through execution feedback
\citep{robeyns2025sica,zhang2025dgm,wang2025hgm,brookes2025artemis}, while
Self-play SWE-RL updates model parameters through training
\citep{wei2026selfplayswerl}. In contrast, \method{} keeps both the LLM and
solver harness fixed and isolates skill-bank evolution, allowing us to study
validation, regression, and convergence over long streams of project-level
software-engineering tasks.
\section{\method{}: Validation-Gated Evolution for Agents}
\label{sec:method}
\label{sec:theory}

\begingroup
\setlength{\abovedisplayskip}{4pt}
\setlength{\belowdisplayskip}{4pt}
\setlength{\abovedisplayshortskip}{2pt}
\setlength{\belowdisplayshortskip}{2pt}

We first formulate long-horizon in-context self-evolution as validation-gated
text-space optimization. Based on that, we then provide theoretical guarantees for performance of
future tasks, including final gain and drawdown
(Sec.~\ref{sec:stability}). We further derive the holdout validation
size required for a target drawdown tolerance (Sec.~\ref{sec:validation}). Full proofs are in
Appendix~\ref{app:convergence}.

\subsection{Problem setup and \method{} framework}
\label{sec:formulation}
\label{sec:valve-framework}

\noindent\textbf{Problem setup.}
We consider an agent that sequentially encounters tasks
$x_1,x_2,\ldots$ from a distribution $\mathcal{D}$ while its model weights and
solver harness remain fixed. Given a task $x$ and textual context $c$, the
frozen policy $\pi$ produces a trajectory
$\tau\sim\pi(\cdot\mid x,c)$ with executable success
$s(x,\tau)\in\{0,1\}$. Its solve probability is
$q(x,c):=\mathbb{E}_{\tau\sim\pi(\cdot\mid x,c)}[s(x,\tau)]$.
The evolving state is a persistent memory bank $\mathcal{B}_t$ of textual artifacts
accumulated from previous tasks.
\newcommand{\figFrameworkOverview}{%
\begin{wrapfigure}{r}{0.4\linewidth}
\centering
\includegraphics[width=\linewidth]{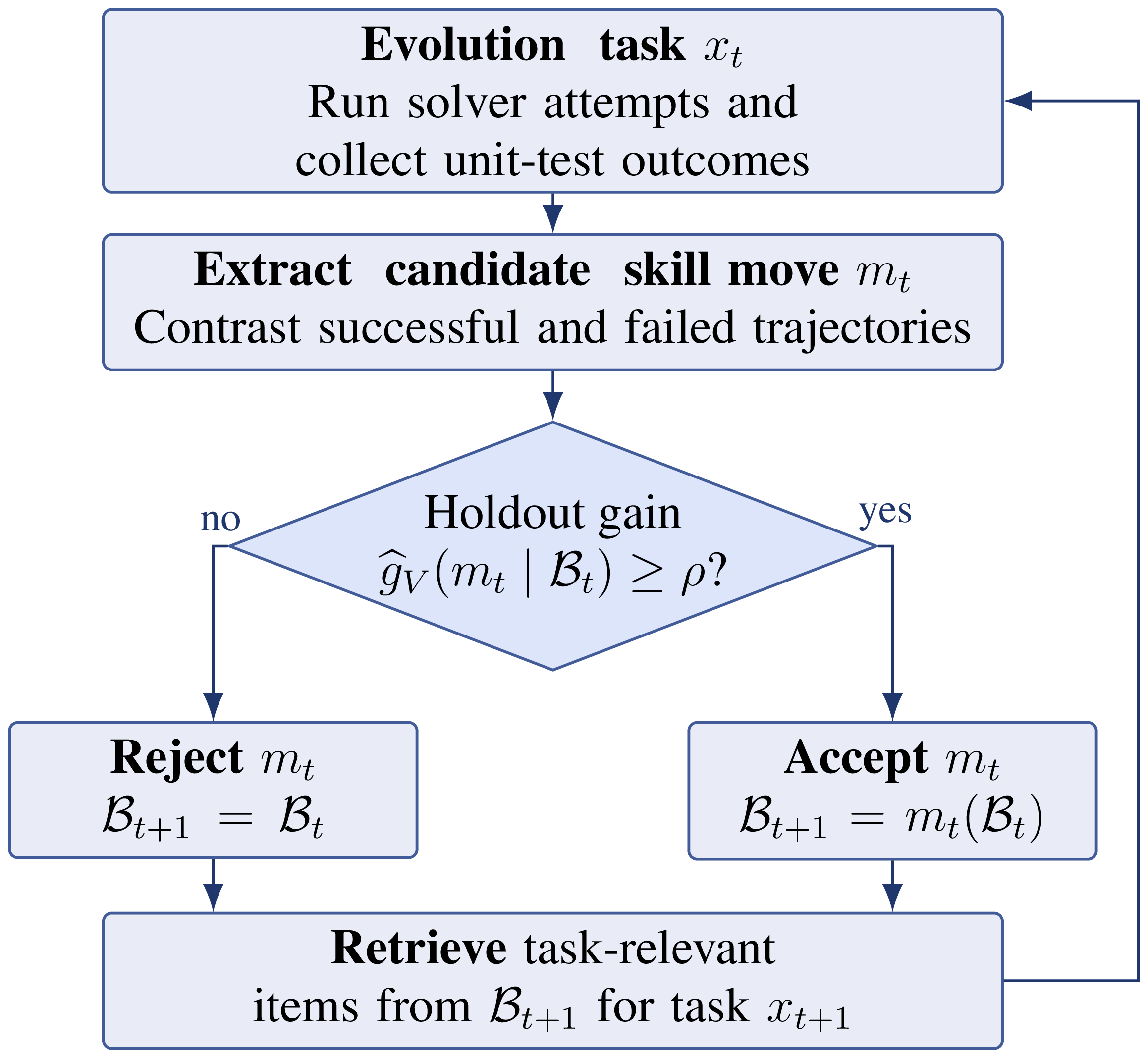}
\caption{Overview of \method{}.}
\label{fig:framework}
\vspace{-1.5\baselineskip}
\end{wrapfigure}}%
\iclronly{\figFrameworkOverview\space}%
Because the full bank may consume the solver's context and even misguide its task solving, a fixed retriever $R$ maps each task and bank into a compact context.
Thus, its population performance is
$F_{\mathcal{D}}(\mathcal{B})
:=\mathbb{E}_{x\sim\mathcal{D}}[q(x,R(x,\mathcal{B}))]$.

\paragraph{A minimal framework for stable evolution.}
\method{} centers evolution around three operators:
\emph{Extract}, \emph{Validate}, and \emph{Retrieve}.
Extract proposes updates from completed experience, Validate decides which updates persist, and
Retrieve maps the persistent bank to a compact task-relevent solver context.
For simplicity, we only consider optimization on Extract and Validate while Retrieve remains fixed.
This minimal interface allows different extraction, validation, memory and retrieval mechanisms to instantiate the same framework.
\autoref{fig:framework} shows the specific instance for skill optimization.

\arxivonly{\figFrameworkOverview}%
\noindent\textbf{Retrieve.}
As Retrieve is fixed in our evolution, we first briefly discuss it.
Before each task, Retrieve selects a compact subset of relevant artifacts from the current
bank. We independently rank principles and skills against the issue and inject the
highest-scoring rich-markdown without re-summarization (Appendix~\ref{app:impl-retrieve}).

\paragraph{Extract.}
On task $x_t$, the solver samples three 
attempts under the same task, repository and retrieved context.
Extract is invoked only when executable outcomes are mixed, enabling a
within-task contrast between successful and failed trajectories. A reflection
call proposes at most three \emph{principles}, which encode transferable
reasoning guidance, and three \emph{skills}, which encode narrow executable
procedures. Exact and near-duplicate items are suppressed before validation
(Appendices~\ref{app:impl-agent}--\ref{app:impl-memory}).

Let $m$ denote a proposed bank update. Its population gain is
$g_{\mathcal{D}}(m\mid\mathcal{B})
=F_{\mathcal{D}}(m(\mathcal{B}))-F_{\mathcal{D}}(\mathcal{B})$.

\paragraph{Validate.}
A lesson that helps its source task may hurt others.
Validate therefore compares the current and candidate banks on a held-out panel
$V=\{x_1,\ldots,x_n\}$ using task-paired executable rollouts.
Let $F_V$ and $g_V$ denote the corresponding panel quantities.
A candidate is committed only when
\begin{equation}
\label{eq:gate}
\mathcal{B}_{t+1}=m_t(\mathcal{B}_t)
\ \ \text{if}\ \
\widehat g_V(m_t\mid\mathcal{B}_t)\ge\rho,
\qquad
\mathcal{B}_{t+1}=\mathcal{B}_t
\ \ \text{otherwise.}
\end{equation}

In our implementation, each proposal jointly updates principles and skills, so
the gate evaluates the complete state transition rather than individual items
(Appendix~\ref{app:impl-validate}).
For banks $\mathcal{B},\mathcal{B}'$ and task $x$, we run $r$ paired rollouts
and record
$\widehat u_{\mathcal{B}',\mathcal{B}}(x)
:=\frac1r\sum_{j\le r}s(x,\tau'_j)
-\frac1r\sum_{j\le r}s(x,\tau_j)\in[-1,1]$.
Finite-rollout noise is absorbed into its variance, bounded by $\sigma^2$ on
$V$ and by $\sigma_{\mathcal{H}}^2$ on a disjoint evaluation set
$\mathcal{H}$.

For an average of $N$ i.i.d.\ summands with variance at most $s^2$ and upper
deviation $Y-\mathbb{E}Y\le M$, the one-sided Bennett radius
$ b_M(N,L,s^2)
:=
\frac{s^2}{M}\,
\mathrm{h}^{-1}\!\Big(\frac{M^2L}{Ns^2}\Big)
\le
\sqrt{\frac{2s^2L}{N}}+\frac{2ML}{3N},
\qquad
\mathrm{h}(v):=(1+v)\ln(1+v)-v$
is exceeded with probability at most $e^{-L}$.

\begin{assumption}[Sample independence]
\label{ass:indep}
For each repository, $\mathcal{D}$ denotes its task distribution. The
validation panel $V$, held-out evaluation set $\mathcal{H}$, and adaptation
stream are disjoint and i.i.d.\ samples from $\mathcal{D}$, and
$\mathcal{H}$ is never used by extraction or validation. Candidates may depend
arbitrarily on previous verdicts from $V$, and $V$ may be reused throughout
evolution.
\end{assumption}

\subsection{Stable evolution and certified performance}
\label{sec:stability}

Here, we show that positive-margin acceptance limits the persistent state from changing indefinitely.

\begin{theorem}[Fixed-panel convergence]
\label{thm:conv-main}
Suppose the panel score of the current bank is carried forward rather than
re-estimated after a rejection, and let $\widehat F_t$ denote the panel score
of $\mathcal{B}_t$ used by \autoref{eq:gate}. Then $\widehat F_t$ is
non-decreasing, $\widehat F_T=\max_t\widehat F_t$, and the number $K$ of
accepted moves satisfies
$
\label{eq:count}
K
\le
\left\lfloor\frac{1-\widehat F_0}{\rho}\right\rfloor
\le
\left\lfloor\frac1\rho\right\rfloor .
$
The run therefore visits at most $K+1$ distinct banks.
\end{theorem}

The result follows directly from positive-margin ascent on a score bounded in
$[0,1]$ and holds for arbitrary history-dependent proposals and reuse of $V$.
If improvement is localized to tasks whose baseline outcomes are unstable,
with fraction $h$, the bound tightens to
$K\le\lfloor h/\rho\rfloor$
(Appendix~\ref{app:preflight-proof}).

As validation convergence does not imply population improvement,
we further evaluate frozen checkpoints on a disjoint set $\mathcal{H}$.
Let $\mathcal{B}^{(0)},\ldots,\mathcal{B}^{(J)}$ be the checkpoints,
$\mathcal{B}^{(J)}$ the final bank, and $A$ the number of distinct banks.
Denote $\widehat\Delta$ for the measured final gain over
$\mathcal{B}^{(0)}$, and let
$
\label{eq:dd}
\mathrm{DD}
:=
\max_j\big[
F_{\mathcal{D}}(\mathcal{B}^{(j)})
-F_{\mathcal{D}}(\mathcal{B}^{(J)})
\big],
\widehat{\mathrm{DD}}
:=
\max_j\widehat D_j .
$

\begin{theorem}[Certified gain and drawdown]
\label{thm:cert-main}
Under Assumption~\ref{ass:indep}, with probability at least $1-\delta$,
\begin{align}
F_{\mathcal{D}}(\mathcal{B}^{(J)})
-F_{\mathcal{D}}(\mathcal{B}^{(0)})
&\ge
\widehat\Delta
-b_1\!\left(
n_{\mathcal{H}},
\ln\tfrac1\delta,
\sigma_{\mathcal{H}}^2
\right),
\label{eq:gaincert}\\[-1pt]
\mathrm{DD}
&\le
\widehat{\mathrm{DD}}
+b_1\!\left(
n_{\mathcal{H}},
\ln\tfrac{A}{\delta},
\sigma_{\mathcal{H}}^2
\right),
\label{eq:ddcert}
\end{align}
where \autoref{eq:gaincert} additionally requires its right-hand side to be
nonnegative.
\end{theorem}

Thus a measured held-out gain certifies expected improvement on future evaluation tasks drawn from $\mathcal{D}$, while \eqref{eq:ddcert} controls peak-to-final population drawdown.
Since Theorem~\ref{thm:conv-main} gives $A\le K+1$, the drawdown budget depends on accepted changes rather than the total task horizon.
A paired version also certifies final separation between gated and
ungated runs (Theorem~\ref{thm:sep-main},
Appendix~\ref{app:separation-proof}).

\subsection{Validation reliability and holdout sizing}
\label{sec:validation}

Because earlier verdicts from the same validation panel can affect later updates, we thus further bound the population loss of adaptive acceptance and then derive the required validation and evaluation set sizes.

\paragraph{Validation reliability.}
For an accepted move, define its population loss as
$\ell_{\mathcal{D}}(m\mid\mathcal{B})
:=[-g_{\mathcal{D}}(m\mid\mathcal{B})]_+$.
Because only accepted moves change the bank, only finite states are reachable under adaptive panel reuse. Conditioning on randomness other than
$V$, the number of bank-move pairs that may produce the $k$-th acceptance
is at most
$
\label{eq:Nk}
N_k:=\sum_{t=1}^{T}\binom{t-1}{k-1},
$
where $T$ is the gate-query horizon fixed in advance.

\begin{theorem}[Adaptive-validation population-loss guarantee]
\label{thm:reuse-main}
Under Assumption~\ref{ass:indep}, let
$\delta_k:=\delta/(k(k+1))$ and
$L_k:=\ln(N_k/\delta_k)$. Then with probability at least $1-\delta$, the
$k$-th accepted move satisfies, simultaneously for every $k$ and every
$\gamma>0$ for which the right-hand side is at least $-\gamma$,

\begin{equation}
\label{eq:reusecert}
g_{\mathcal{D}}(m\mid\mathcal{B})
\ge
\widehat g_V(m\mid\mathcal{B})
-b_{1+\gamma}(n,L_k,\sigma^2)
\ge
\rho-b_{1+\gamma}(n,L_k,\sigma^2).
\end{equation}
\end{theorem}

The guarantee is anchored at the measured validation gain and remains valid
despite adaptive reuse of the same panel $V$.

\paragraph{Holdout sizing.}
All precedding certificates have the measured quantity where the radius decreases with the holdout size.
Inverting this radius therefore gives the number of held-out tasks required for
a target tolerance.

\begin{corollary}[Holdout sizing]
\label{cor:sizing}
Fix $\gamma>0$ and write $a:=2s^2L$ and $c:=2ML/3$. Then
$b_M(N,L,s^2)\le\gamma$ whenever

\begin{equation}
\label{eq:sizing}
N
\ge
\left(
\frac{2c}{\sqrt{a+4c\gamma}-\sqrt{a}}
\right)^{\!2}
=
\frac{a}{\gamma^2}\big(1+O(\gamma)\big).
\end{equation}
\end{corollary}
The corollary applies to the two set sizing separately.
Setting $L=\ln(A/\delta)$ and $M=1$ sizes the evaluation set for a prescribed
drawdown slack $\gamma$.
For the validation panel, using $L=L_K$, $M=1+\gamma$, and replacing
$\gamma$ by $\widehat g_V+\gamma$ ensures that every
accepted update in a run with at most $K$ accepts has population loss at most
$\gamma$.

The bank count satisfies $A\le\lfloor1/\rho\rfloor+1$ by
Theorem~\ref{thm:conv-main}. Baseline rollouts provide a variance proxy.
If $p(x)$ is the baseline solve probability, Assumption~\ref{ass:local} gives
$\sigma^2\le(1+\tfrac2r)\mathbb{E}_x[p(x)(1-p(x))]
+2\varepsilon_{\mathrm{var}}$, whose leading term is estimated by
$
\label{eq:vmass}
\mathrm{vmass}
:=
\frac1N\sum_{i=1}^N
\widehat p_i(1-\widehat p_i)\frac{r}{r-1}.
$
One can therefore estimate the variance from baseline rollouts, choose
$\rho$, $\gamma$, and a conservative residual allowance
$\varepsilon_{\mathrm{var}}$, and obtain both holdout sizes from
\autoref{eq:sizing}. Numerical examples are provided in
Table~\ref{tab:sizing}, Figure~\ref{fig:sizing}, and
Appendix~\ref{app:holdout-choice}.

\endgroup
\section{Experiments}
\label{sec:experiments}

\subsection{Experimental setup}
\label{sec:exp-setup}

We implement \method{} using mini-SWE-agent \citep{minisweagent2026,yang2024sweagent}, where the solver,
extractor, and consumer all are based on the same backbone, using four SWE-smith
repositories \citep{yang2025swesmith}: \texttt{pygments}, \texttt{sqlfluff},
\texttt{pydicom}, and \texttt{deepdiff}. Each repository is split into disjoint
partitions (Appendix~\ref{app:impl-reference}): a $100$-task evaluation set, a
fixed $50$-task validation panel reused for every acceptance decision, and an
evolution stream of approximately $1{,}000$ tasks with at most $100$
mixed-outcome evolution steps. All conditions use three rollouts per task,
rich-markdown banks, and the acceptance threshold $\rho=0.02$.

\begin{table}[t]
\centering
\footnotesize
\renewcommand{\arraystretch}{1.08}

\begin{tabular}{@{}lccccc@{}}
\toprule
& \multicolumn{5}{c}{\textbf{Final resolved rate (avg@3, \%)}} \\
\cmidrule(lr){2-6}
\textbf{Setting}
& \textbf{\texttt{pygments}}
& \textbf{\texttt{sqlfluff}}
& \textbf{\texttt{deepdiff}}
& \textbf{\texttt{pydicom}}
& \textbf{Model Avg.} \\
\midrule

\multicolumn{6}{@{}l}{\textbf{MiniMax-M2.7}} \\

Base
& $52.5$
& $66.5$
& $56.3$
& $60.3$
& $58.9$ \\

SkillOpt
& $60.0\ {\color{gainpos}(+7.5)}$
& $66.0\ {\color{gainneg}(-0.5)}$
& $\mathbf{69.0}\ {\color{gainpos}\mathbf{(+12.7)}}$
& $60.7\ {\color{gainpos}(+0.4)}$
& $63.9\ {\color{gainpos}(+5.0)}$ \\

\method{} w/o val
& $53.3\ {\color{gainpos}(+0.8)}$
& $66.7\ {\color{gainpos}(+0.2)}$
& $61.7\ {\color{gainpos}(+5.4)}$
& $58.7\ {\color{gainneg}(-1.6)}$
& $60.1\ {\color{gainpos}(+1.2)}$ \\

\cellcolor{locushl!55}\method{}
& \cellcolor{locushl!55}$\mathbf{61.0}\ {\color{gainpos}\mathbf{(+8.5)}}$
& \cellcolor{locushl!55}$\mathbf{72.7}\ {\color{gainpos}\mathbf{(+6.2)}}$
& \cellcolor{locushl!55}$65.3\ {\color{gainpos}(+9.0)}$
& \cellcolor{locushl!55}$\mathbf{61.7}\ {\color{gainpos}\mathbf{(+1.4)}}$
& \cellcolor{locushl!55}$\mathbf{65.2}\ {\color{gainpos}\mathbf{(+6.3)}}$ \\

\addlinespace[3pt]

\multicolumn{6}{@{}l}{\textbf{GPT-5.5}} \\

Base
& $59.3$
& $68.7$
& $64.7$
& $38.2$
& $57.7$ \\

SkillOpt
& $58.3\ {\color{gainneg}(-1.0)}$
& $72.0\ {\color{gainpos}(+3.3)}$
& $68.7\ {\color{gainpos}(+4.0)}$
& $56.7\ {\color{gainpos}(+18.5)}$
& $63.9\ {\color{gainpos}(+6.2)}$ \\

\method{} w/o val
& $60.3\ {\color{gainpos}(+1.0)}$
& $69.7\ {\color{gainpos}(+1.0)}$
& $65.7\ {\color{gainpos}(+1.0)}$
& $68.0\ {\color{gainpos}(+29.8)}$
& $65.9\ {\color{gainpos}(+8.2)}$ \\

\cellcolor{locushl!55}\method{}
& \cellcolor{locushl!55}$\mathbf{62.3}\ {\color{gainpos}\mathbf{(+3.0)}}$
& \cellcolor{locushl!55}$\mathbf{83.3}\ {\color{gainpos}\mathbf{(+14.6)}}$
& \cellcolor{locushl!55}$\mathbf{85.0}\ {\color{gainpos}\mathbf{(+20.3)}}$
& \cellcolor{locushl!55}$\mathbf{76.3}\ {\color{gainpos}\mathbf{(+38.1)}}$
& \cellcolor{locushl!55}$\mathbf{76.7}\ {\color{gainpos}\mathbf{(+19.0)}}$ \\

\addlinespace[3pt]

\multicolumn{6}{@{}l}{\textbf{Claude-4.6}} \\

Base
& $78.5$
& $79.3$
& $71.7$
& $71.3$
& $75.2$ \\

SkillOpt
& $74.7\ {\color{gainneg}(-3.8)}$
& $88.7\ {\color{gainpos}(+9.4)}$
& $80.0\ {\color{gainpos}(+8.3)}$
& $\mathbf{98.0}\ {\color{gainpos}\mathbf{(+26.7)}}$
& $85.3\ {\color{gainpos}(+10.2)}$ \\

\method{} w/o val
& $80.0\ {\color{gainpos}(+1.5)}$
& $62.0\ {\color{gainneg}(-17.3)}$
& $67.3\ {\color{gainneg}(-4.4)}$
& $65.3\ {\color{gainneg}(-6.0)}$
& $68.7\ {\color{gainneg}(-6.6)}$ \\

\cellcolor{locushl!55}\method{}
& \cellcolor{locushl!55}$\mathbf{92.3}\ {\color{gainpos}\mathbf{(+13.8)}}$
& \cellcolor{locushl!55}$\mathbf{92.3}\ {\color{gainpos}\mathbf{(+13.0)}}$
& \cellcolor{locushl!55}$\mathbf{97.7}\ {\color{gainpos}\mathbf{(+26.0)}}$
& \cellcolor{locushl!55}$96.7\ {\color{gainpos}(+25.4)}$
& \cellcolor{locushl!55}$\mathbf{94.8}\ {\color{gainpos}\mathbf{(+19.6)}}$ \\

\midrule

\multicolumn{6}{@{}l}{\textbf{Repo average}} \\

Base
& $63.4$
& $71.5$
& $64.2$
& $56.6$
& $63.9$ \\

SkillOpt
& $64.3\ {\color{gainpos}(+0.9)}$
& $75.6\ {\color{gainpos}(+4.1)}$
& $72.6\ {\color{gainpos}(+8.3)}$
& $71.8\ {\color{gainpos}(+15.2)}$
& $71.1\ {\color{gainpos}(+7.1)}$ \\

\method{} w/o val
& $64.6\ {\color{gainpos}(+1.1)}$
& $66.1\ {\color{gainneg}(-5.4)}$
& $64.9\ {\color{gainpos}(+0.7)}$
& $64.0\ {\color{gainpos}(+7.4)}$
& $64.9\ {\color{gainpos}(+0.9)}$ \\

\cellcolor{locushl!55}\method{}
& \cellcolor{locushl!55}$\mathbf{71.9}\ {\color{gainpos}\mathbf{(+8.4)}}$
& \cellcolor{locushl!55}$\mathbf{82.8}\ {\color{gainpos}\mathbf{(+11.3)}}$
& \cellcolor{locushl!55}$\mathbf{82.7}\ {\color{gainpos}\mathbf{(+18.4)}}$
& \cellcolor{locushl!55}$\mathbf{78.2}\ {\color{gainpos}\mathbf{(+21.6)}}$
& \cellcolor{locushl!55}$\mathbf{78.9}\ {\color{gainpos}\mathbf{(+14.9)}}$ \\

\bottomrule
\end{tabular}

\caption{
Final performance after long-horizon SWE-smith repository-level evolution.
Parentheses show gains over the corresponding base model; best results are bold.
}
\label{tab:main}
\end{table}
\begin{wraptable}{r}{0.43\linewidth}
  \centering
  \vspace{-0.5\baselineskip}
  \footnotesize
  \setlength{\tabcolsep}{2.0pt}
  \renewcommand{\arraystretch}{1.05}

  \begin{tabular}{@{}lrrr@{}}
  \toprule
  \textbf{Setting}
  & \shortstack{\textbf{Final}\\\textbf{gain}}
  & \shortstack{\textbf{Peak}\\\textbf{gain}}
  & \shortstack{\textbf{Draw-}\\\textbf{down}} \\
  \midrule

  \multicolumn{4}{@{}l}{\textbf{MiniMax-M2.7}} \\
  SkillOpt
  & $+5.0$
  & $+7.4$
  & $\mathbf{2.3}$ \\
  \method{} w/o val
  & $+1.2$
  & $+5.4$
  & $4.2$ \\
  \cellcolor{locushl!55}\method{}
  & \cellcolor{locushl!55}$\mathbf{+6.3}$
  & \cellcolor{locushl!55}$\mathbf{+9.0}$
  & \cellcolor{locushl!55}$2.7$ \\

  \addlinespace[1.5pt]

  \multicolumn{4}{@{}l}{\textbf{GPT-5.5}} \\
  SkillOpt
  & $+6.2$
  & $+12.8$
  & $6.6$ \\
  \method{} w/o val
  & $+8.2$
  & $+10.1$
  & $1.9$ \\
  \cellcolor{locushl!55}\method{}
  & \cellcolor{locushl!55}$\mathbf{+19.0}$
  & \cellcolor{locushl!55}$\mathbf{+19.4}$
  & \cellcolor{locushl!55}$\mathbf{0.4}$ \\

  \addlinespace[1.5pt]

  \multicolumn{4}{@{}l}{\textbf{Claude-4.6}} \\
  SkillOpt
  & $+10.2$
  & $+12.1$
  & $2.0$ \\
  \method{} w/o val
  & $-6.5$
  & $+5.6$
  & $12.1$ \\
  \cellcolor{locushl!55}\method{}
  & \cellcolor{locushl!55}$\mathbf{+19.6}$
  & \cellcolor{locushl!55}$\mathbf{+21.1}$
  & \cellcolor{locushl!55}$\mathbf{1.5}$ \\

  \midrule

  \multicolumn{4}{@{}l}{\textbf{Overall}} \\
  SkillOpt
  & $+7.1$
  & $+10.8$
  & $3.6$ \\
  \method{} w/o val
  & $+0.9$
  & $+7.0$
  & $6.1$ \\
  \cellcolor{locushl!55}\method{}
  & \cellcolor{locushl!55}$\mathbf{+14.9}$
  & \cellcolor{locushl!55}$\mathbf{+16.5}$
  & \cellcolor{locushl!55}$\mathbf{1.5}$ \\

  \bottomrule
  \end{tabular}

  \caption{
  Gain retention over evolution.
  Best results are bold.
  }
  \label{tab:stability}
  \vspace{-1.0\baselineskip}
\end{wraptable}
We compare \method{} against the recent strong skill-optimization  method SkillOpt~\citep{yang2026skillopt},
\method{} without validation, and the unevolved base model.
For a controlled comparison, SkillOpt and \method{}
both use $50$ validation tasks per repository
We evaluate three frontier backone LLMs spanning different model families and capabilty levels: the open-weight MiniMax-M2.7
\citep{minimax2026m27}, served locally with temperature $1.0$, top-$p$ $0.95$,
and top-$k$ $40$, and two proprietary models, GPT-5.5
\citep{openai2026gpt55} and Claude-4.6-sonnet
\citep{anthropic2026claude46}.

As the goal is to properly evaluate the long-horizon agent evolution, we use final-checkpoint avg@3 resolved rate on the held-out evaluation set as
the primary performance metric (Table~\ref{tab:main}). 
To measure retention, we additionally report peak gain and
peak-to-final drawdown over the evolution horizon
(Table~\ref{tab:stability}); complete per-repository peak results are in
Appendix~\ref{app:impl-evaluation}.
Applying the theorical bounds with the stated task-level variance bound 
$\sigma^2=0.025$, the $100$-task evaluation set yields one-sided radii of $4.8$
points for final gain and $6.5$ points for drawdown at $A=8$
(Appendix~\ref{app:holdout-choice}). Validated evolution takes $122.1$ hours per
run on average, compared with $50.6$ hours without the gate and $96.1$ hours for
SkillOpt (Appendix~\ref{app:efficiency}).

\subsection{Main results}
\label{sec:main-results}

Table~\ref{tab:main} shows that \method{} improves final performance across all
three backbones. Across the twelve model-repository runs, the average resolved
rate increases from $63.9\%$ to $78.9\%$, a gain of $14.9$ points, and every
\method{} run outperforms both its unevolved base model and its counterpart without validation.
The model-average gains are $6.3, 19.0, 19.6$ for MiniMax-M2.7, GPT-5.5, and Claude-4.6, respectively.
\method{} also outperforms SkillOpt in final model average for all three backbones,
although SkillOpt remains best on MiniMax-M2.7 \texttt{deepdiff} and
Claude-4.6 \texttt{pydicom}.
These results are consistent with our theoretical analysis and support validaton-gated updates as a practical mechanism for reliable long-horizon agent evolution.

\paragraph{Validation turns transient gains into retained gains.}
Table~\ref{tab:stability} separates peak improvement from the gain retained at
the end of evolution. Overall, SkillOpt reaches a peak gain of $10.8$ points
and retains $7.1$, with $3.6$ points of drawdown. \method{} w/o val reaches
$7.0$ points at peak but retains only $0.9$, with $6.1$ points of drawdown.
\method{} instead reaches $16.5$ points at peak and retains $14.9$, reducing
average drawdown to $1.5$ points.

The contrast is strongest for Claude-4.6: without validation, the average run
peaks at $+5.6$ points but ends at $-6.5$, whereas \method{} reaches $+21.1$ and
retains $+19.6$. Three of the four Claude-4.6 runs without validation and the
MiniMax-M2.7 \texttt{pydicom} run finish below baseline.

Figure~\ref{fig:validation-effect} shows the same retention effect over time.
On \texttt{pygments} with Claude-4.6
(Figure~\ref{fig:evolution-curve-main}), \method{} peaks at $97.0\%$ and ends at
$92.3\%$ after approximately $1{,}000$ stream tasks, whereas \method{} w/o val
peaks at $91.7\%$ and falls to $80.0\%$, an $11.7$-point drawdown. Across all
twelve runs (Figure~\ref{fig:drawdown-main}), validation reduces average
drawdown from $6.1$ to $1.5$ points.

\section{Ablation and Insights for Skill Optimization}
\label{sec:ablation}

We analyze which design choices drive skill optimization
(Sec.~\ref{sec:ablation-study}), how validation size affects retained gains
(Sec.~\ref{sec:validation-size}), and how validation shapes skill-bank growth
and utility (Sec.~\ref{sec:bank-growth}).

\subsection{Component ablations}
\label{sec:ablation-study}

\newcommand{\figAblationBody}[1]{%
\centering
\includegraphics[width=#1\linewidth]{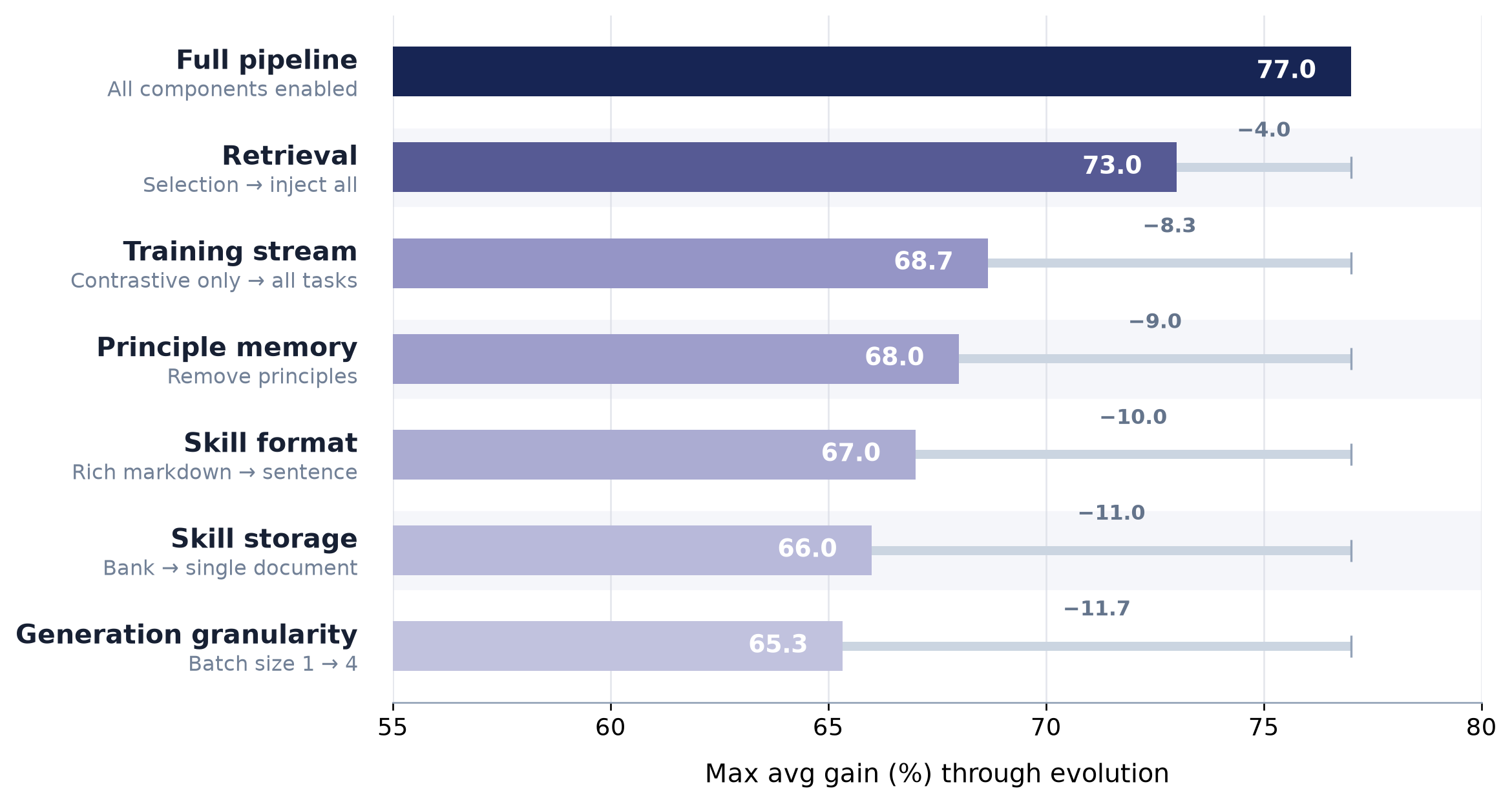}
\caption{
Ablation of component-level design choices on \texttt{pydicom} with GPT-5.5.
Appendix~\ref{app:ablation-design} maps prior self-evolving methods to the compared alternatives; the step-0 baseline is $38.2\%$.
}
\label{fig:ablation}}
\iclronly{\begin{figure}[ht]\figAblationBody{0.98}\end{figure}}
\arxivonly{\begin{figure}[H]\figAblationBody{0.75}\end{figure}}

Figure~\ref{fig:ablation} compares the full pipeline with six variants, each
changing one component while holding the others fixed. Both sides of every
comparison correspond to design choices used in existing self-evolving methods
(Table~\ref{tab:ablation-design} in Appendix~\ref{app:ablation-design} maps each
choice to prior work). The full pipeline outperforms every variant. Selective
retrieval contributes $4.0$ points, while restricting extraction to the
contrastive training stream contributes $8.3$ points. Removing principle memory
costs $9.0$ points, and replacing rich markdown with sentence-format skills
costs $10.0$ points. The largest drops come from replacing the skill bank with a
single document ($11.0$ points) and increasing generation batch size from $1$
to $4$ ($11.7$ points), indicating that skill organization and extraction
granularity have the largest measured effects.

\subsection{Effect of validation size}
\label{sec:validation-size}

\newcommand{\figValidationSize}[1]{%
\begin{wrapfigure}{r}{#1\linewidth}
  \centering
  \vspace{-1.5\baselineskip}
  \includegraphics[width=\linewidth]{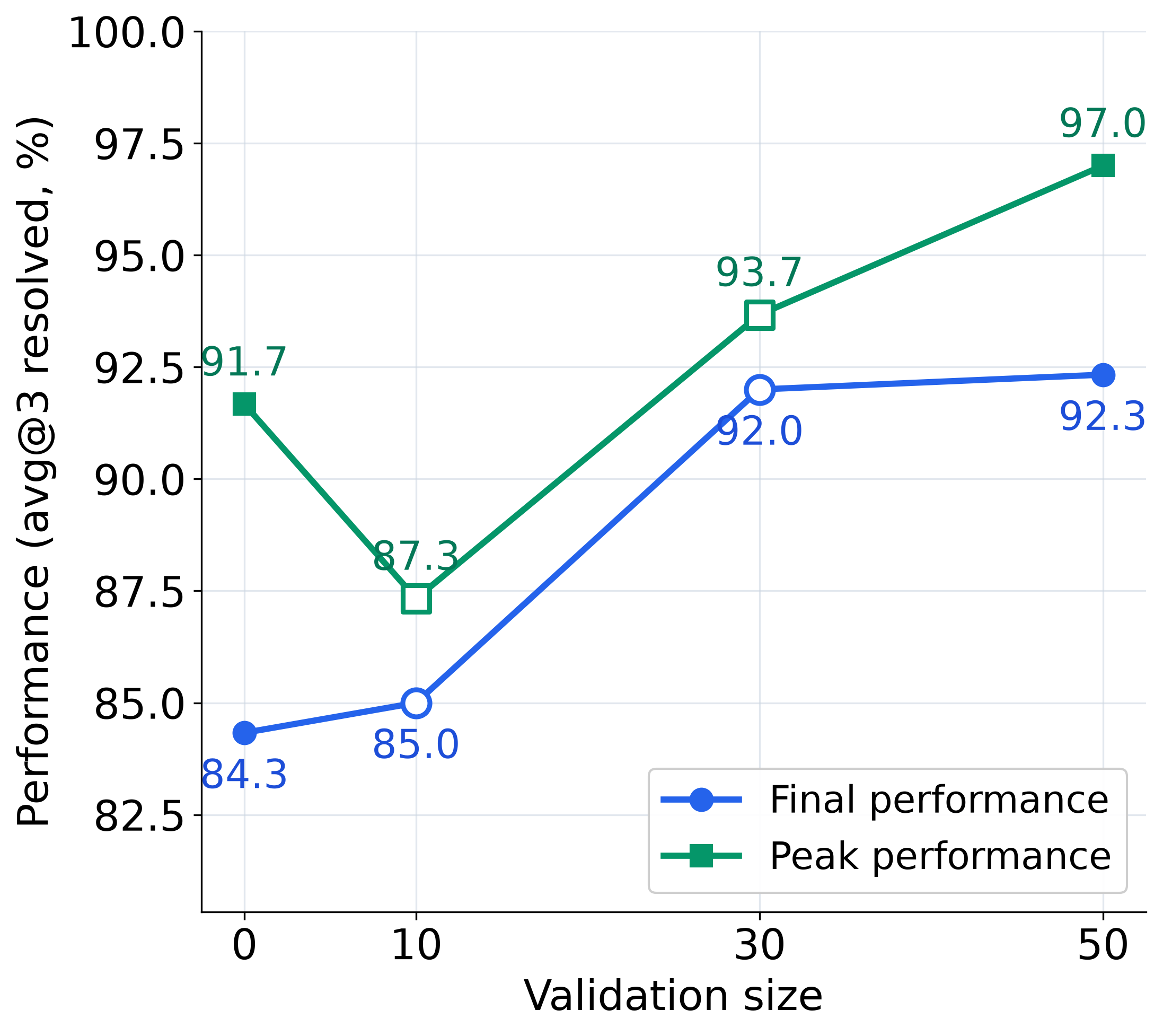}
  \caption{Final and peak avg@3 resolved rate versus validation sizes using 
  Claude-4.6 on repository \texttt{pygments}.}
  \label{fig:validation-size}
  \vspace{-1.0\baselineskip}
\end{wrapfigure}}
\newcommand{\figBankGrowth}[2]{
\begin{wrapfigure}{r}{#1\linewidth}
  \centering
  \vspace{-0.5\baselineskip}
  \includegraphics[width=\linewidth]{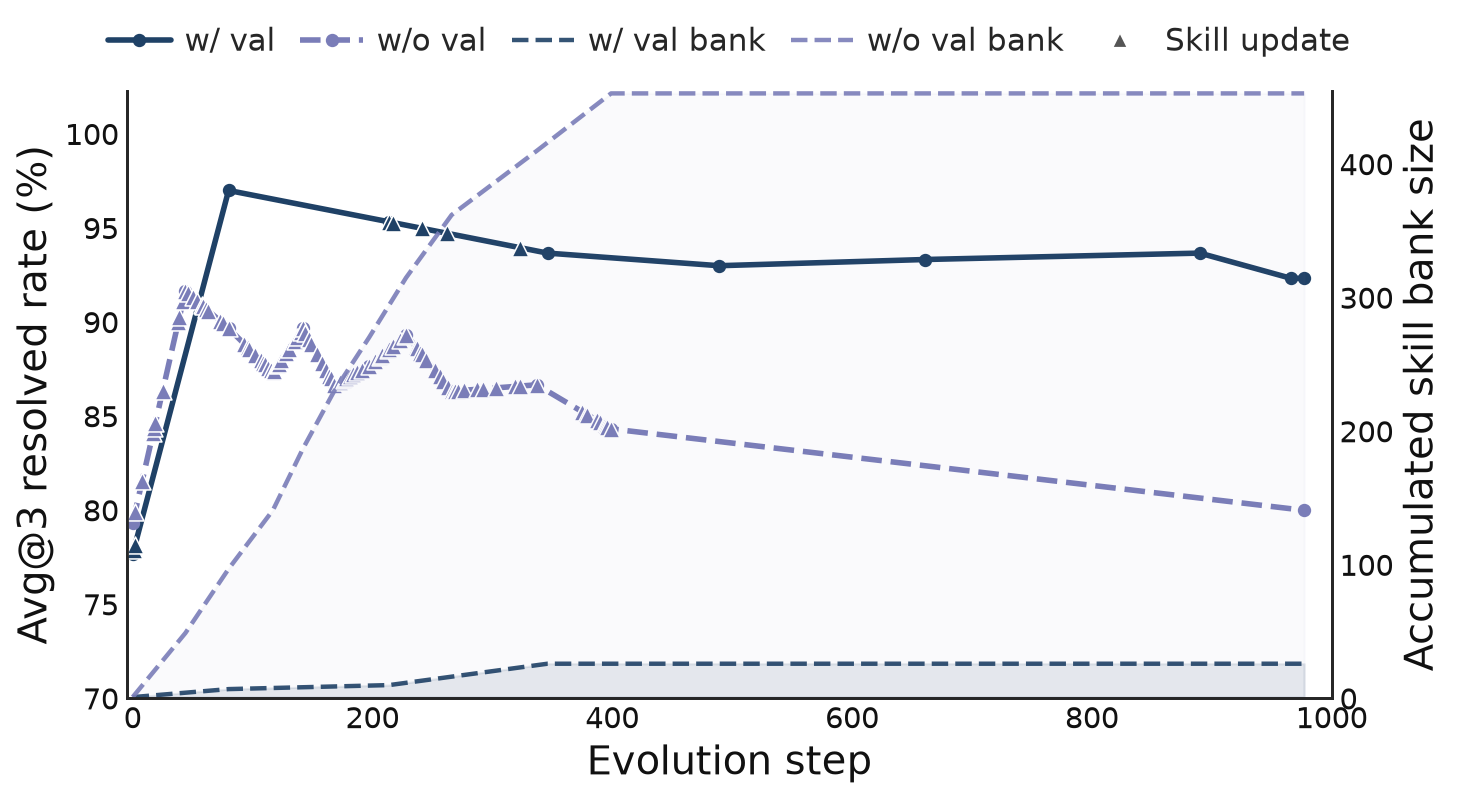}
  \caption{Accumulated bank size over evolution for
  Claude-4.6 on \texttt{pygments}. Solid curves show avg@3 resolved rate,
  dashed curves show bank size, and triangles mark bank updates.}
  \label{fig:bank-growth-example}
  \vspace{#2\baselineskip}
\end{wrapfigure}}

\iclronly{\figValidationSize{0.43}}
\arxivonly{\figValidationSize{0.40}}

Figure~\ref{fig:validation-size} sweeps the validation size over $0$, $10$,
$30$, and $50$ tasks for Claude-4.6 on \texttt{pygments}. Final resolved rate
rises from $80.0\%$ without validation to $85.0\%$, $92.0\%$, and $92.3\%$ at
sizes $10$, $30$, and $50$, with most of the improvement occurring between $10$
and $30$ tasks. Peak performance is not monotone: it first falls from $91.7\%$
to $87.3\%$, then rises to $93.7\%$ and $97.0\%$.

The final-performance trend is qualitatively consistent with
Sec.~\ref{sec:theory}: for fixed variance and log-budget, a larger panel shrinks
the deviation radius of Theorem~\ref{thm:reuse-main}, tightening the certified bound on 
per-update population loss.
The theorem does not imply monotone realized performance,
however, because changing the panel also changes the accepted updates and the resulting evolution path. 
We therefore interpret this sweep as evidence
that validation size affects retained gains, not as empirical verification of the
bound or evidence that any particular panel size is universally sufficient.

\subsection{Validated skill-bank dynamics and utility}
\label{sec:bank-growth}
\label{sec:case-study}
\begin{table}[t]
\centering
\footnotesize
\setlength{\tabcolsep}{3.0pt}
\renewcommand{\arraystretch}{1.05}

\begin{tabular}{@{}lccccc@{}}
\toprule
& \multicolumn{5}{c}{\textbf{Final bank size (skills / principles)}} \\
\cmidrule(lr){2-6}
\textbf{Setting}
& \textbf{\texttt{pygments}}
& \textbf{\texttt{sqlfluff}}
& \textbf{\texttt{deepdiff}}
& \textbf{\texttt{pydicom}}
& \textbf{Total} \\
\midrule

\multicolumn{6}{@{}l}{\textbf{MiniMax-M2.7}} \\

\method{} w/o val
& $244/231$
& $237/220$
& $238/239$
& $205/197$
& $924/887$ \\

\cellcolor{locushl!55}\method{}
& \cellcolor{locushl!55}$30/34$
& \cellcolor{locushl!55}$38/38$
& \cellcolor{locushl!55}$14/13$
& \cellcolor{locushl!55}$34/31$
& \cellcolor{locushl!55}$116/116$ \\

\addlinespace[3pt]

\multicolumn{6}{@{}l}{\textbf{GPT-5.5}} \\

\method{} w/o val
& $249/194$
& $234/160$
& $267/210$
& $174/127$
& $924/691$ \\

\cellcolor{locushl!55}\method{}
& \cellcolor{locushl!55}$3/3$
& \cellcolor{locushl!55}$3/3$
& \cellcolor{locushl!55}$3/2$
& \cellcolor{locushl!55}$38/28$
& \cellcolor{locushl!55}$47/36$ \\

\addlinespace[3pt]

\multicolumn{6}{@{}l}{\textbf{Claude-4.6}} \\

\method{} w/o val
& $249/204$
& $227/184$
& $54/33$
& $44/25$
& $574/446$ \\

\cellcolor{locushl!55}\method{}
& \cellcolor{locushl!55}$12/14$
& \cellcolor{locushl!55}$11/10$
& \cellcolor{locushl!55}$5/4$
& \cellcolor{locushl!55}$13/12$
& \cellcolor{locushl!55}$41/40$ \\

\midrule

\multicolumn{6}{@{}l}{\textbf{Repo total}} \\

\method{} w/o val
& $742/629$
& $698/564$
& $559/482$
& $423/349$
& $2{,}422/2{,}024$ \\

\cellcolor{locushl!55}\method{}
& \cellcolor{locushl!55}$45/51$
& \cellcolor{locushl!55}$52/51$
& \cellcolor{locushl!55}$22/19$
& \cellcolor{locushl!55}$85/71$
& \cellcolor{locushl!55}$204/192$ \\

\bottomrule
\end{tabular}

\caption{
Final bank size after repository-level evolution.
Entries report skills / principles.
}
\label{tab:bank-size}
\end{table}

Validation substantially changes the dynamics of bank growth. On
\texttt{pygments} with Claude-4.6, the validated bank grows through seven
accepted updates to $26$ items ($12$ skills and $14$ principles) by step $323$
and then remains fixed. In contrast, \method{} w/o val accumulates $453$ items
by approximately step $400$ yet finishes at a lower resolved rate
($80.0\%$ versus $92.3\%$; Figure~\ref{fig:bank-growth-example}).
This suggests
that long-horizon improvement depends not on continually accumulating more
memory, but on selectively retaining useful updates.

Across all main runs, \method{} retains $204$ skills and $192$ principles,
compared with $2{,}422$ skills and $2{,}024$ principles for
\method{} w/o val, an $11\times$ reduction in total bank size
(Table~\ref{tab:bank-size}).

\iclronly{\figBankGrowth{0.48}{-1.5}}%
\arxivonly{\figBankGrowth{0.52}{-0.5}}%

The retained items are also individually useful. We evaluate each of the $26$
items in the validated \texttt{pygments} bank alone on the same $100$-task
holdout with three rollouts per task. Every item improves resolved rate over the
no-memory baseline, with gains ranging from $0.33$ to $10.33$ points
(Table~\ref{tab:artifact-utility} in Appendix~\ref{app:cases}). We group the
items into \emph{Environment}, \emph{Repair Strategy}, and \emph{Code
Semantics}, reflecting whether they concern execution, the repair process, or
program behavior.

The learned items span multiple levels of abstraction
(Figure~\ref{fig:artifact-examples}). Skills encode executable procedures,
ranging from a Pygments-specific yield-tuple invariant to a general workflow
for multi-edit regressions, while principles capture transferable diagnostic
lessons, such as restoring intended statement order instead of masking the
error with an arbitrary default.

\begin{figure}[t]
\centering
\setlength{\fboxsep}{0pt}
\setlength{\fboxrule}{0.5pt}
\newcommand{\artifactExamplesBox}{%
\fcolorbox{black!30}{white}{%
\begin{minipage}{0.95\linewidth}
{\setlength{\fboxsep}{6pt}%
\colorbox{black!7}{\parbox{\dimexpr\linewidth-12pt\relax}{%
\strut\small\textbf{Representative learned items}}}}\par
\vspace{8pt}

\hspace*{10pt}\begin{minipage}{\dimexpr\linewidth-20pt\relax}
\small

\hyperref[app:item-audit-yield]{%
\nolinkurl{audit_yield_tuple_index_group_consistency.md}}\par
{\footnotesize\color{black!55}
Code Semantics $\cdot$ Skill $\cdot$ evolved steps 323 $\cdot$ gain $+2.00$ pp.}
\par\smallskip
Summarized content: Check every lexer yield tuple to ensure that its token
position and value come from the same regular-expression capture group.\par

\vspace{5pt}
{\color{black!15}\hrule height 0.4pt}
\vspace{5pt}

\hyperref[app:item-enumerate-swaps]{%
\nolinkurl{enumerate_all_swapped_values_in_buggy_commit.md}}\par
{\footnotesize\color{black!55}
Repair Strategy $\cdot$ Skill $\cdot$ evolved steps 262 $\cdot$ gain $+10.33$ pp.}
\par\smallskip
Summarized content: Inspect the complete bug-introducing commit, classify every
mutation, and fix all checklist items rather than only the first visible swap.
\par

\vspace{5pt}
{\color{black!15}\hrule height 0.4pt}
\vspace{5pt}

\hyperref[app:item-use-before-assignment]{%
\nolinkurl{variable_use_before_assignment_signals_reordered_code.md}}\par
{\footnotesize\color{black!55}
Code Semantics $\cdot$ Principle $\cdot$ evolved steps 323 $\cdot$ gain $+3.33$ pp.}
\par\smallskip
Summarized content: When a previously working function uses a variable before
its later assignment, restore the intended statement order instead of adding
an arbitrary default.

\end{minipage}
\vspace{9pt}
\end{minipage}}}
\iclronly{\artifactExamplesBox}%
\providecommand{\artifactscale}{0.90}
\arxivonly{\scalebox{\artifactscale}{\artifactExamplesBox}}%

\caption{
Representative learned skills and principles.
Linked Appendix entries contain the full stored contents.
}
\label{fig:artifact-examples}
\end{figure}
\section{Conclusion}

Long-term self-evolution asks how an agent with a frozen LLM core can
continually improve at test time, and its central obstacle is unstable updates
that cause drawdown and eventual collapse over long horizons. We distilled the
common design of existing in-context evolution methods into an
Extract--Validate--Retrieve pipeline and formulated it as text-space
optimization of a retrieved skill bank. On this formulation, positive-margin
acceptance converges after finitely many bank changes, a disjoint evaluation set
certifies final gain and long-horizon drawdown, and a uniform pathwise bound
controls the population loss of updates chosen adaptively on a reused
validation panel; inverting the same deviation radius gives sizing rules for
both holdouts. Empirically, \method sustains evolution over streams of
approximately $1{,}000$ software engineering tasks per repository, achieving
average final and peak gains of $14.9$ and $16.5$ points across three backbones
and four repositories, compared with $0.9$ and $7.0$ points without a gate,
while cutting average drawdown by $4.0\times$ with an $11\times$ more compact
skill bank. As agents move toward lifelong deployment, this work offers
practical guidance on validation gating and a step toward stable, trustworthy
continual self-evolution.
\markmainend

\bibliographystyle{plainnat}
\bibliography{main}

\newpage
\appendix

\section{Proofs and Certification Details}
\label{app:convergence}

This appendix proves the results of Sec.~\ref{sec:theory}. Appendix~\ref{app:radius}
establishes the deviation radius and the range constants used throughout;
Appendix~\ref{app:conv} proves convergence and Appendix~\ref{app:preflight-proof}
the localized refinement; Appendices~\ref{app:cert-proof} and
\ref{app:separation-proof} prove the held-out certificates;
Appendix~\ref{app:reuse-proof} proves the reused-panel bound; and
Appendix~\ref{app:sizing-proof} inverts the radius.

\subsection{Online dynamics of the context policy}
\label{app:online}

The solve probability of a context policy $\mu$ marginalizes the context:
\begin{equation}
  \label{eq:marginal}
  \mathbb{P}_{\mu}(\tau\in\mathcal{T}_x^+\mid x)
  =\sum_c\underbrace{q(x,c)}_{\text{quality}}
  \underbrace{\mu(c\mid x,\mathcal{E})}_{\text{production probability}}
  =\mathbb{E}_{c\sim\mu(\cdot\mid x,\mathcal{E})}[q(x,c)],
\end{equation}
where $\mathcal{T}_x^+:=\{\tau:s(x,\tau)=1\}$. The two factors expose the central
design questions behind the population objective $F_{\mathcal{D}}$ of Sec.~\ref{sec:formulation}: which contexts have high quality, and
how should the agent produce them?

Experience is endogenous: the policy shapes the rollouts that subsequently
change the policy. The agent acts causally. At step $t$, it observes only the
strictly past experience $\mathcal{E}_{<t}$, draws
$c_t\sim\mu(\cdot\mid x_t,\mathcal{E}_{<t})$, produces
$\tau_t\sim\pi(\cdot\mid x_t,c_t)$, and observes $s_t:=s(x_t,\tau_t)$. It then
appends the complete interaction record,
\begin{equation}
  \label{eq:experience-update}
  \mathcal{E}_{<t+1}=\mathcal{E}_{<t}\mathbin{\|}(x_t,c_t,\tau_t,s_t),
\end{equation}
where $\mathbin{\|}$ denotes sequence concatenation. Thus, although
$F_{\mathcal{D}}$ (Sec.~\ref{sec:formulation}) is written as a population objective, its context policy and
experience distribution are coupled through the online rollout process. Over an
evolution horizon of length $T$, the corresponding cumulative objective is
\begin{equation}
  \label{eq:online-objective}
  \mathcal{J}_T(\mu):=\mathbb{E}\!\left[\sum_{t=1}^{T}s_t\right]
  =\mathbb{E}\!\left[\sum_{t=1}^{T}q(x_t,c_t)\right],
\end{equation}
where the expectation includes task sampling, context selection, and trajectory
generation. A successful evolution rule must therefore improve future contexts
without using information from future interactions.

\subsection{The deviation radius}
\label{app:radius}

\begin{lemma}[One-sided Bennett radius]
\label{lem:bennett}
Let $Y_1,\ldots,Y_N$ be i.i.d.\ with mean $\nu$, variance at most $s^2$, and
$Y_i-\nu\le M$ almost surely. Then
\begin{equation}
  \label{eq:lem-bennett}
  \mathbb{P}\left(\frac1N\sum_{i=1}^N Y_i-\nu>b_M(N,L,s^2)\right)\;\le\;e^{-L},
\end{equation}
where $b_M$ is the one-sided Bennett radius of Sec.~\ref{sec:formulation},
\begin{equation}
  \label{eq:radius}
  b_M(N,L,s^2)
  :=\frac{s^2}{M}\,
  \mathrm{h}^{-1}\!\Big(\frac{M^2L}{Ns^2}\Big)
  \le\sqrt{\frac{2s^2L}{N}}+\frac{2ML}{3N},
  \qquad
  \mathrm{h}(v):=(1+v)\ln(1+v)-v .
\end{equation}
\end{lemma}

\begin{proof}
Bennett's inequality for independent summands with the stated variance and
one-sided range gives, for every $b>0$,
\begin{equation}
  \label{eq:bennett-raw}
  \mathbb{P}\left(\frac1N\sum_{i=1}^N Y_i-\nu>b\right)
  \;\le\;\exp\left(-\frac{Ns^2}{M^2}\,\mathrm{h}\!\left(\frac{Mb}{s^2}\right)\right),
  \qquad \mathrm{h}(v)=(1+v)\ln(1+v)-v .
\end{equation}
The function $\mathrm{h}$ is strictly increasing on $[0,\infty)$ with
$\mathrm{h}(0)=0$, so it is invertible there. Setting the exponent in
\eqref{eq:bennett-raw} equal to $-L$ and solving,
\begin{equation}
  \frac{Ns^2}{M^2}\,\mathrm{h}\!\left(\frac{Mb}{s^2}\right)=L
  \quad\Longleftrightarrow\quad
  \frac{Mb}{s^2}=\mathrm{h}^{-1}\!\left(\frac{M^2L}{Ns^2}\right)
  \quad\Longleftrightarrow\quad
  b=\frac{s^2}{M}\,\mathrm{h}^{-1}\!\left(\frac{M^2L}{Ns^2}\right),
\end{equation}
which is the first expression in \eqref{eq:radius}.

For the displayed relaxation, use $\mathrm{h}(v)\ge v^2/(2+\tfrac23v)$, valid for
$v\ge0$. Substituting $v=Mb/s^2$ turns \eqref{eq:bennett-raw} into
\begin{equation}
  \exp\left(-\frac{Ns^2}{M^2}\cdot\frac{M^2b^2/s^4}{2+\tfrac{2Mb}{3s^2}}\right)
  =\exp\left(-\frac{Nb^2}{2s^2+\tfrac{2M}{3}b}\right).
\end{equation}
Setting this to $e^{-L}$ gives $Nb^2=L\big(2s^2+\tfrac{2M}{3}b\big)$, that is
$b^2\le a'+c'b$ with $a'=2s^2L/N$ and $c'=2ML/(3N)$. Since $b^2\le a'+c'b$
implies $b\le\sqrt{a'}+c'$, we obtain
$b\le\sqrt{2s^2L/N}+2ML/(3N)$.
\end{proof}

Every summand below lies in $[-1,1]$, for which $M=2$ is always admissible. The
following remark identifies the smaller admissible range in each of our
statements.

\begin{remark}[Effective range]
\label{rem:range}
Each certificate has the form: establish $\nu\ge c$ for a population contrast
$\nu$, by showing that the empirical mean does not exceed $\nu$ by more than $b$,
with $c$ equal to the empirical value minus $b$. On the failure event $\nu<c$,
the deviation required is
\begin{equation}
  \label{eq:peel}
  \frac1N\sum_i Y_i-\nu\;\ge\;b+(c-\nu),
\end{equation}
which grows as $\nu$ decreases, while the admissible range $M=1-\nu$ also grows.
Substituting both into the exponent of \eqref{eq:bennett-raw} gives
\begin{equation}
  \frac{Ns^2}{(1-\nu)^2}\,\mathrm{h}\!\left(\frac{(1-\nu)(b+c-\nu)}{s^2}\right).
\end{equation}
For completeness, this exponent is nondecreasing in $-\nu$ whenever, for every
$\nu\in[-1,c]$,
\begin{equation}
  \frac{(1-\nu)(b+c+1-2\nu)}{s^2}\ln(1+v_\nu)
  \;\ge\;2\,\mathrm{h}(v_\nu),
  \qquad
  v_\nu:=\frac{(1-\nu)(b+c-\nu)}{s^2}.
  \label{eq:effective-range-condition}
\end{equation}
This follows by differentiating the displayed exponent with respect to
$1-\nu$ and is directly checkable for any reported $(N,s^2,b,c)$. We verified
\eqref{eq:effective-range-condition} over the full interval for the numerical
configurations used in this paper. Under this condition the binding case is
$\nu\to c$ and the effective range is $M=1-c$. Consequently: for the drawdown
\eqref{eq:ddcert}, only checkpoints with a positive drop can attain the maximum,
so $c=0$ and $M=1$; for the gain \eqref{eq:gaincert} and the separation
\eqref{eq:sepcert}, $c\ge0$ whenever the certified floor is nonnegative, so
$M=1$; and for the reused panel \eqref{eq:reusecert}, $c=-\gamma$ and
$M=1+\gamma$. If \eqref{eq:effective-range-condition} is not verified, the
same statements remain valid with the conservative range $M=2$, because every
summand lies in $[-1,1]$.
\end{remark}

\subsection{Proof of fixed-panel convergence}
\label{app:conv}

\begin{proof}[Proof of Theorem~\ref{thm:conv-main}]
A rejected move leaves both the bank and its recorded panel score unchanged, so
$\widehat F_{t+1}=\widehat F_t$. An accepted move at step $t$ replaces the current
bank by $m_t(\mathcal{B}_t)$ and, because the accepted candidate's score is
carried forward, replaces $\widehat F_t$ by the candidate's score. Hence
\begin{equation}
  \label{eq:increment}
  \widehat F_{t+1}-\widehat F_t=\widehat g_V(m_t\mid\mathcal{B}_t)\;\ge\;\rho,
\end{equation}
the inequality being the acceptance rule \eqref{eq:gate}. In both cases
$\widehat F_{t+1}\ge\widehat F_t$, so the sequence is non-decreasing and
$\widehat F_T=\max_t\widehat F_t$.

Let $t_1<\cdots<t_K$ be the accepted rounds. Summing \eqref{eq:increment} over
them and telescoping,
\begin{equation}
  \widehat F_{t_K+1}-\widehat F_0
  =\sum_{j=1}^{K}\big(\widehat F_{t_j+1}-\widehat F_{t_j}\big)
  =\sum_{j=1}^{K}\widehat g_V(m_{t_j}\mid\mathcal{B}_{t_j})
  \;\ge\;K\rho .
\end{equation}
Since $\widehat F_t$ is an average of quantities in $[0,1]$ it satisfies
$\widehat F_t\in[0,1]$, whence
\begin{equation}
  K\rho\;\le\;\widehat F_{t_K+1}-\widehat F_0\;\le\;1-\widehat F_0\;\le\;1 ,
\end{equation}
and dividing by $\rho$ and taking integer parts gives \eqref{eq:count}. The bank
changes only at accepted rounds, so the number of distinct banks visited is at
most $K+1$.

The argument is deterministic. It uses no property of $V$ beyond $F_V$ being
bounded, and in particular permits arbitrary reuse of $V$ and history-dependent
proposals.
\end{proof}

\begin{remark}
If the current bank were re-scored with fresh rollouts at every comparison
instead of carried forward, \eqref{eq:increment} would hold only up to the
re-estimation drift. On the event that every such drift is bounded by
$\eta<\rho/2$, the same telescoping gives
$K\le\lfloor1/(\rho-2\eta)\rfloor$.
\end{remark}

\subsection{Baseline-concentrated headroom}
\label{app:preflight-proof}

Let $p(x)$ denote the baseline solve probability, let
$U:=\{x:p(x)\in(0,1)\}$ be the baseline-unstable set, and let
$h:=\mathbb{P}_{x\sim\mathcal{D}}(x\in U)$. With only $r$ baseline rollouts,
define the observed unstable panel fraction
$\widehat h:=n^{-1}\sum_{i=1}^n\mathbf{1}\{0<\widehat p_i<1\}$. This is generally
downward biased for $h$: tasks with small but nonzero $p(x)$ can produce only
failures in all $r$ rollouts.

\begin{assumption}[Baseline-concentrated gain]
\label{ass:local}
There are residual allowances $\varepsilon_{\mathrm{loc}},
\varepsilon_{\mathrm{var}}\ge0$ such that every reachable bank satisfies
\begin{equation}
  F_V(\mathcal{B})-F_V(\mathcal{B}_0)
  \le \widehat h+\varepsilon_{\mathrm{loc}},
\end{equation}
and every reachable move satisfies
\begin{align}
  \mathbb{E}_x[u_{m,\mathcal{B}}(x)^2]
  &\le \mathbb{E}_x[p(x)(1-p(x))]+\varepsilon_{\mathrm{var}},\\
  \mathbb{E}_x[\operatorname{Var}(\widehat u\mid x)]
  &\le \frac2r\mathbb{E}_x[p(x)(1-p(x))]+\varepsilon_{\mathrm{var}}.
\end{align}
\end{assumption}

Under Assumption~\ref{ass:local}, the number of accepted moves satisfies the
localized bound
\begin{equation}
  \label{eq:Kbound}
  K\rho\;\le\;\widehat h+\varepsilon_{\mathrm{loc}},
  \qquad\text{i.e.,}\qquad
  K\;\le\;\left\lfloor\frac{\widehat h+\varepsilon_{\mathrm{loc}}}{\rho}\right\rfloor .
\end{equation}

\begin{proof}[Proof of \eqref{eq:Kbound}]
Assumption~\ref{ass:local} directly bounds the realizable panel headroom. Using
it in the telescoping step of Appendix~\ref{app:conv} gives
$K\rho\le\widehat h+\varepsilon_{\mathrm{loc}}$, which is
\eqref{eq:Kbound}. The idealized zero-residual result is recovered at
$\varepsilon_{\mathrm{loc}}=0$, but we do not impose that value on the
experiments.
\end{proof}

The variance part of the assumption and the law of total variance give
\begin{equation}
  \operatorname{Var}\big[\widehat u\big]
  =\operatorname{Var}_x\big[u\big]+\mathbb{E}_x\big[\operatorname{Var}(\widehat u\mid x)\big]
  \;\le\;\Big(1+\frac2r\Big)\mathbb{E}_x\big[p(x)(1-p(x))\big]
  +2\varepsilon_{\mathrm{var}}.
\end{equation}
For $r$ independent Bernoulli baseline rollouts on task $i$, the plug-in
estimator satisfies
\begin{equation}
  \mathbb{E}\big[\widehat p_i(1-\widehat p_i)\big]=\frac{r-1}{r}\,p_i(1-p_i),
\end{equation}
so multiplying by $r/(r-1)$ removes the finite-$r$ downward bias, and averaging
over tasks gives \eqref{eq:vmass} as an unbiased estimate of
$\mathbb{E}_x[p(1-p)]$. When $r=3$ each empirically unstable task has
$\widehat p\in\{1/3,2/3\}$ and contributes $1/3$, so empirically
$\mathrm{vmass}=\widehat h/3$. This identity does not make $\widehat h$ an
unbiased estimate of the population mass $h$.

\subsection{Proof of the certified gain and drawdown}
\label{app:cert-proof}

\begin{proof}[Proof of Theorem~\ref{thm:cert-main}]
We prove the two statements separately.

\emph{Final gain.} The reporting rule fixes the pair
$(\mathcal{B}^{(0)},\mathcal{B}^{(J)})$ before $\mathcal{H}$ is drawn: both banks
are measurable functions of the adaptation stream, the validation panel and the
remaining randomness of the run, none of which involve $\mathcal{H}$ under
Assumption~\ref{ass:indep}. Conditioning on all of that randomness, the per-task
summands
\begin{equation}
  Y_i:=\widehat u_{\mathcal{B}^{(J)},\mathcal{B}^{(0)}}(x'_i),
  \qquad x'_i\in\mathcal{H},
\end{equation}
are i.i.d.\ in $[-1,1]$ with
\begin{equation}
  \mathbb{E}[Y_i]=F_{\mathcal{D}}(\mathcal{B}^{(J)})-F_{\mathcal{D}}(\mathcal{B}^{(0)})=:\nu,
  \qquad \operatorname{Var}[Y_i]\le\sigma_{\mathcal{H}}^2 ,
\end{equation}
and their average is $\widehat\Delta$. Applying Lemma~\ref{lem:bennett} with
$L=\ln(1/\delta)$ and, by Remark~\ref{rem:range}, $M=1$ whenever the certified
floor is nonnegative,
\begin{equation}
  \mathbb{P}\Big(\widehat\Delta-\nu>b_1\big(n_{\mathcal{H}},\ln\tfrac1\delta,\sigma_{\mathcal{H}}^2\big)\Big)\le\delta ,
\end{equation}
and on the complementary event $\nu\ge\widehat\Delta-b_1$, which is
\eqref{eq:gaincert}. No union bound arises because the comparison is fixed before
$\mathcal{H}$ is seen.

\emph{Drawdown.} Let
$D_j:=F_{\mathcal{D}}(\mathcal{B}^{(j)})-F_{\mathcal{D}}(\mathcal{B}^{(J)})$ with
paired estimate $\widehat D_j$ on $\mathcal{H}$. Checkpoints that hold the same
bank give identical $D_j$ and identical $\widehat D_j$, so it suffices to index
the $A$ distinct banks. Each pair $(\mathcal{B}^{(j)},\mathcal{B}^{(J)})$ is
again fixed before $\mathcal{H}$ is drawn, because the checkpoint schedule is
chosen in advance and the banks are functions of the stream and the panel.
Applying Lemma~\ref{lem:bennett} to each with $L=\ln(A/\delta)$ and $M=1$ by
Remark~\ref{rem:range}, and taking a union bound over the $A$ indices,
\begin{equation}
  \mathbb{P}\Big(\exists j:\;D_j>\widehat D_j+b_1\big(n_{\mathcal{H}},\ln\tfrac{A}{\delta},\sigma_{\mathcal{H}}^2\big)\Big)
  \;\le\;A\cdot\frac{\delta}{A}=\delta .
\end{equation}
On the complementary event, taking the maximum over $j$ on both sides,
\begin{equation}
  \mathrm{DD}=\max_j D_j\;\le\;\max_j\widehat D_j+b_1
  =\widehat{\mathrm{DD}}+b_1 ,
\end{equation}
which is \eqref{eq:ddcert}. That $M=1$ is admissible follows from
Remark~\ref{rem:range} with $c=0$: since $D_J=0$, both $\mathrm{DD}\ge0$ and
$\widehat{\mathrm{DD}}\ge0$, so indices with $D_j\le0$ satisfy the conclusion
trivially and only $D_j>0$ needs covering. Finally, $A\le K+1$ by
Theorem~\ref{thm:conv-main}.
\end{proof}

\subsection{Certified separation between gated and ungated evolution}
\label{app:separation-proof}

Write $\mathcal{B}^{\mathrm{val}}$ and $\mathcal{B}^{\mathrm{noval}}$ for the
final banks produced on a given repository with and without the gate.

\begin{theorem}[Certified separation of retained performance]
\label{thm:sep-main}
Fix a repository with task distribution $\mathcal{D}$, evaluate both runs on the
same held-out set $\mathcal{H}$, and let $\sigma_W^2$ bound the variance of the
per-task cross-condition contrast
\begin{equation}
  \label{eq:W}
  W(x):=q\big(x,R(x,\mathcal{B}^{\mathrm{val}})\big)
       -q\big(x,R(x,\mathcal{B}^{\mathrm{noval}})\big).
\end{equation}
Then with probability at least $1-\delta$,
\begin{equation}
  \label{eq:sepcert}
  F_{\mathcal{D}}(\mathcal{B}^{\mathrm{val}})
  -F_{\mathcal{D}}(\mathcal{B}^{\mathrm{noval}})
  \;\ge\;\widehat W-b_1\big(n_{\mathcal{H}},\ln\tfrac1\delta,\sigma_W^2\big).
\end{equation}
\end{theorem}

\begin{corollary}[Repository average]
\label{cor:sep-agg}
If the statement is made for $G$ repositories with disjoint evaluation sets of
$n_{\mathcal{H}}$ tasks each, then with probability at least $1-\delta$ the
average of the $G$ repository-level differences exceeds the average of the
$\widehat W_r$ by at most
$b_1\big(G\,n_{\mathcal{H}},\ln\tfrac1\delta,\sigma_W^2\big)$.
\end{corollary}

The contrast is paired: both runs are scored on the same tasks, so task
difficulty cancels within each task and one radius bounds the difference.

\begin{proof}[Proof of Theorem~\ref{thm:sep-main}]
Both final banks are measurable functions of their own adaptation streams, the
validation panel of the gated run, and the remaining randomness, none of which
involve $\mathcal{H}$ under Assumption~\ref{ass:indep}. Conditioning on that
randomness, the per-task values
\begin{equation}
  W(x'_i)=q\big(x'_i,R(x'_i,\mathcal{B}^{\mathrm{val}})\big)
         -q\big(x'_i,R(x'_i,\mathcal{B}^{\mathrm{noval}})\big),
  \qquad x'_i\in\mathcal{H},
\end{equation}
are i.i.d.\ in $[-1,1]$ with mean
$F_{\mathcal{D}}(\mathcal{B}^{\mathrm{val}})-F_{\mathcal{D}}(\mathcal{B}^{\mathrm{noval}})$
and variance at most $\sigma_W^2$, and their empirical average is $\widehat W$.
Lemma~\ref{lem:bennett} with $L=\ln(1/\delta)$ and $M=1$ gives
\eqref{eq:sepcert}.

For the corollary, index repositories by $r=1,\ldots,G$. Their evaluation sets
are disjoint, so the per-task values are independent across repositories as well
as within them, and the grand average
\begin{equation}
  \frac1G\sum_{r=1}^{G}\widehat W_r
  =\frac{1}{Gn_{\mathcal{H}}}\sum_{r=1}^{G}\sum_{i=1}^{n_{\mathcal{H}}}W_r(x'_{r,i})
\end{equation}
is an average of $Gn_{\mathcal{H}}$ i.i.d.\ bounded summands whose mean is the
average of the $G$ repository-level differences. Applying
Lemma~\ref{lem:bennett} with $N=Gn_{\mathcal{H}}$ gives the claim.
\end{proof}

\subsection{Proof of the reused-panel population-loss guarantee}
\label{app:reuse-proof}

Condition on the randomness $\xi$ of the run other than the validation panel: the
adaptation task stream, solver sampling, extraction sampling and retrieval
tie-breaking. For an accept pattern $S\subseteq[T]$ define recursively
\begin{equation}
  \label{eq:recursion}
  \mathcal{B}(S,1)=\mathcal{B}_0,\qquad
  m(S,t)=\mathrm{Prop}\big(\xi,\mathcal{B}(S,t),t\big),\qquad
  \mathcal{B}(S,t+1)=
  \begin{cases}
    m(S,t)\big(\mathcal{B}(S,t)\big), & t\in S,\\
    \mathcal{B}(S,t), & t\notin S .
  \end{cases}
\end{equation}

\begin{lemma}[Path compression]
\label{lem:compress}
For every $S$ and $t$: (i) $\mathcal{B}(S,t)$ depends on $S$ only through
$S\cap[t-1]$; (ii) the pair $(\mathcal{B}(S,t),m(S,t))$ is $\xi$-measurable;
(iii) the realized run is the path of $S^\ast=\{\text{accepted steps}\}$; and
(iv) for each $k\ge1$ the family
\begin{equation}
  \mathcal{F}_k(\xi):=\Big\{\big(\mathcal{B}(S,t),m(S,t)\big):
  t\in[T_{\mathrm{gate}}],\;S\subseteq[T_{\mathrm{gate}}],\;t\in S,
  \;|S\cap[t-1]|=k-1\Big\}
\end{equation}
of pairs that can occur at an accepted step preceded by exactly $k-1$ accepts
satisfies $|\mathcal{F}_k(\xi)|\le N_k$ with $N_k$ as in \eqref{eq:Nk}.
\end{lemma}

\begin{proof}
(i) The recursion \eqref{eq:recursion} consults $S$ only at indices strictly
below $t$, so $\mathcal{B}(S,t)$ is unchanged if $S$ is modified at indices
$\ge t$. (ii) By induction on $t$: $\mathcal{B}(S,1)=\mathcal{B}_0$ is fixed, and
if $\mathcal{B}(S,t)$ is $\xi$-measurable then so is
$m(S,t)=\mathrm{Prop}(\xi,\mathcal{B}(S,t),t)$, the proposal mechanism being a
deterministic function of $\xi$, the current bank and the step index; hence
$\mathcal{B}(S,t+1)$ is $\xi$-measurable. (iii) Immediate from the definition of
the gate \eqref{eq:gate}, which accepts exactly at the steps in $S^\ast$.
(iv) A pair in $\mathcal{F}_k$ is determined by the step $t$ at which it occurs
together with the prefix $S\cap[t-1]$, by (i). The constraint $t\in S$ and
$|S\cap[t-1]|=k-1$ means the prefix is a subset of $[t-1]$ of size exactly
$k-1$, of which there are $\binom{t-1}{k-1}$. Summing over
$t\in[T_{\mathrm{gate}}]$ gives
\begin{equation}
  |\mathcal{F}_k(\xi)|\;\le\;\sum_{t=1}^{T_{\mathrm{gate}}}
  \binom{t-1}{k-1}=N_k . \qedhere
\end{equation}
\end{proof}

The family $\mathcal{F}_k(\xi)$ is defined without reference to the panel: it
enumerates what the optimizer could have produced under every possible sequence
of verdicts. Each of its elements is therefore a fixed bank-and-move pair against
which the panel is an honest i.i.d.\ sample.

\begin{proof}[Proof of Theorem~\ref{thm:reuse-main}]
Fix $k\ge1$ and work conditionally on $\xi$. Let $(\mathcal{B},m)\in
\mathcal{F}_k(\xi)$. By Lemma~\ref{lem:compress}(ii) it is $\xi$-measurable and
hence, by Assumption~\ref{ass:indep}, independent of the panel. The panel values
\begin{equation}
  Y_i:=\widehat u_{m(\mathcal{B}),\mathcal{B}}(x_i),\qquad x_i\in V,
\end{equation}
are therefore i.i.d.\ in $[-1,1]$ with mean $g_{\mathcal{D}}(m\mid\mathcal{B})$
and variance at most $\sigma^2$, and their average is
$\widehat g_V(m\mid\mathcal{B})$. Lemma~\ref{lem:bennett} with $L=L_k$ and, by
Remark~\ref{rem:range}, $M=1+\gamma$, gives
\begin{equation}
  \mathbb{P}\Big(\widehat g_V(m\mid\mathcal{B})-g_{\mathcal{D}}(m\mid\mathcal{B})
  >b_{1+\gamma}(n,L_k,\sigma^2)\Big)\;\le\;e^{-L_k}=\frac{\delta_k}{N_k}.
\end{equation}
A union bound over the at most $N_k$ elements of $\mathcal{F}_k(\xi)$ gives
\begin{equation}
  \label{eq:stratum}
  \mathbb{P}\Big(\exists(\mathcal{B},m)\in\mathcal{F}_k(\xi):\;
  \widehat g_V-g_{\mathcal{D}}>b_{1+\gamma}(n,L_k,\sigma^2)\Big)\;\le\;\delta_k .
\end{equation}
By Theorem~\ref{thm:conv-main} the run accepts at most
$\lfloor1/\rho\rfloor$ moves surely, so only finitely many strata are non-empty,
and summing \eqref{eq:stratum} over $k\ge1$,
\begin{equation}
  \sum_{k\ge1}\delta_k=\delta\sum_{k\ge1}\frac{1}{k(k+1)}
  =\delta\sum_{k\ge1}\left(\frac1k-\frac1{k+1}\right)=\delta .
\end{equation}
Hence with probability at least $1-\delta$ the event in \eqref{eq:stratum} fails
for every $k$ simultaneously. On that event, the $k$-th accepted move of the
realized run lies in $\mathcal{F}_k(\xi)$ by Lemma~\ref{lem:compress}(iii)--(iv),
so
\begin{equation}
  g_{\mathcal{D}}(m\mid\mathcal{B})
  \;\ge\;\widehat g_V(m\mid\mathcal{B})-b_{1+\gamma}(n,L_k,\sigma^2)
  \;\ge\;\rho-b_{1+\gamma}(n,L_k,\sigma^2),
\end{equation}
the last step by the acceptance rule \eqref{eq:gate}. The bound holds
conditionally on $\xi$ with probability at least $1-\delta$ for every $\xi$, and
therefore unconditionally.
\end{proof}

The proof uses no independence between candidates, no freshness of the panel and
no restriction on how proposals react to past verdicts; reusing the rollouts of
the unchanged current bank induces dependence across the elements of
$\mathcal{F}_k(\xi)$, which a union bound tolerates. It does require
$T_{\mathrm{gate}}$ to be fixed in advance, since it determines $N_k$.

\subsection{Proof of the sizing rule}
\label{app:sizing-proof}

\begin{proof}[Proof of Corollary~\ref{cor:sizing}]
By Lemma~\ref{lem:bennett} it suffices to make the displayed relaxation of
$b_M$ at most $\gamma$, that is
\begin{equation}
  \sqrt{\frac{a}{N}}+\frac{c}{N}\;\le\;\gamma,
  \qquad a:=2s^2L,\quad c:=\frac{2ML}{3}.
\end{equation}
Substituting $u:=N^{-1/2}>0$ turns this into the quadratic condition
\begin{equation}
  c\,u^2+\sqrt{a}\,u-\gamma\;\le\;0 ,
\end{equation}
whose left-hand side is increasing in $u$ on $u>0$ and vanishes at
\begin{equation}
  u^\star=\frac{-\sqrt a+\sqrt{a+4c\gamma}}{2c}.
\end{equation}
The condition is therefore $u\le u^\star$, i.e.\ $N\ge(u^\star)^{-2}$, which is
\eqref{eq:sizing}. Expanding $\sqrt{a+4c\gamma}=\sqrt a\,(1+2c\gamma/a+O(\gamma^2))$
gives $u^\star=\gamma/\sqrt a+O(\gamma^2)$ and hence
$N=(2s^2L/\gamma^2)(1+O(\gamma))$, the stated leading order. All reported sizes invert the
exact radius $b_M$ rather than its relaxation, and are therefore no larger.
\end{proof}

\subsection{Numerical instantiation}
\label{app:holdout-choice}

\begin{figure}[t]
  \centering
  \begin{subfigure}[t]{0.48\linewidth}
    \centering
    \includegraphics[width=\linewidth]{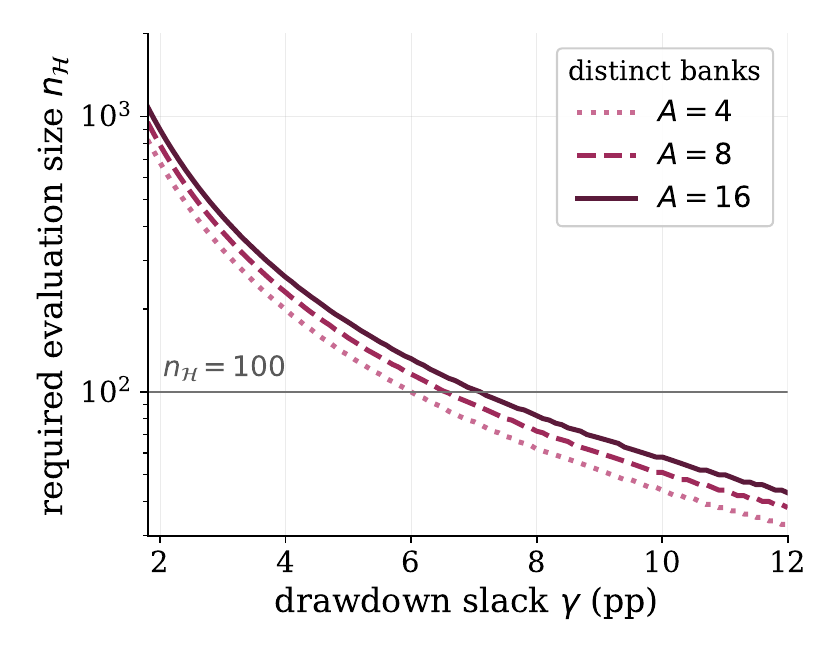}
    \caption{Sizing the evaluation set.}
    \label{fig:sizing-eval}
  \end{subfigure}
  \hfill
  \begin{subfigure}[t]{0.48\linewidth}
    \centering
    \includegraphics[width=\linewidth]{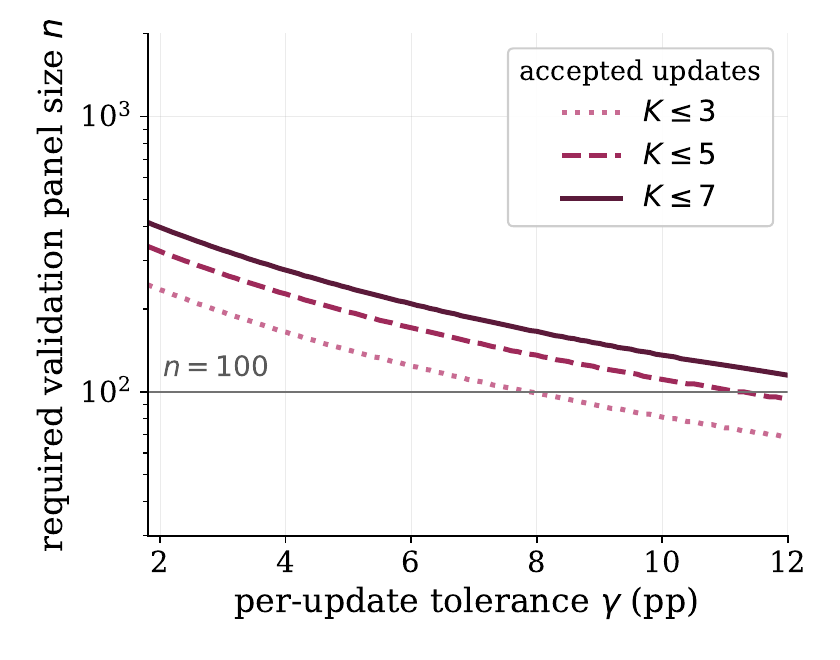}
    \caption{Sizing the validation panel.}
    \label{fig:sizing-panel}
  \end{subfigure}
  \caption{Holdout size required for a prescribed tolerance, from
  Corollary~\ref{cor:sizing}. The vertical axes are logarithmic; horizontal lines mark
  the $100$-task evaluation set and the $50$-task validation panel used in our
  runs.}
  \label{fig:sizing}
\end{figure}

\begin{table}[t]
\centering\small
\setlength{\tabcolsep}{7pt}
\begin{tabular}{@{}lrrrrrr@{}}
\toprule
& \multicolumn{3}{c}{Evaluation set $n_{\mathcal{H}}$}
& \multicolumn{3}{c}{Validation panel $n$}\\
\cmidrule(lr){2-4}\cmidrule(l){5-7}
Tolerance $\gamma$ & $A=4$ & $A=8$ & $A=16$
& $K\le3$ & $K\le5$ & $K\le7$\\
\midrule
$2$pp & $680$ & $787$ & $895$ & $235$ & $324$ & $396$\\
$3$pp & $328$ & $380$ & $432$ & $195$ & $268$ & $327$\\
$5$pp & $136$ & $157$ & $179$ & $142$ & $195$ & $239$\\
$8$pp & $62$ & $72$ & $82$ & $99$ & $136$ & $166$\\
$10$pp & $44$ & $51$ & $58$ & $81$ & $111$ & $136$\\
\bottomrule
\end{tabular}
\caption{Holdout size required for a prescribed tolerance, at $\delta=0.05$
and $\sigma^2=\sigma^2_{\mathcal{H}}=0.025$. Left: tasks needed so that the
certified population drawdown exceeds the measured one by at most $\gamma$, for
$A$ distinct banks among the checkpoints. Right: tasks needed so that every
accepted update of a run making at most $K$ of them costs no more than
$\gamma$, at a gate-query horizon $T=60$ and anchored at the observed accepted
panel gain of $5.8$ points.}
\label{tab:sizing}
\end{table}

Parameters: $\rho=0.02$, $\delta=0.05$, $r=3$, $n=50$, $n_{\mathcal{H}}=100$, and $G=4$
repositories with disjoint evaluation sets. The protocol caps gate queries at one
per triggering task with at most $100$ such tasks; across the main runs the
gate was queried $492$ times in total, an average of $61.5$ per run with a
maximum of $85$, and a horizon of $T=60$ is used for prospective sizing. The
realized number of accepted updates averages $4.75$ with a maximum of $K=7$, so
the localized bound $\lfloor h/\rho\rfloor=10$ of \eqref{eq:Kbound} is not
binding. Accepted updates clear the threshold with a run-level minimum averaging
$5.8$ points and a run-level median of $7.1$ points.

Theorem~\ref{thm:conv-main} also gives the sharper run-specific cap
$\lfloor(1-\widehat F_0)/\rho\rfloor$. The formal cap uses the initial
validation-panel score of each run. For scale, substituting the held-out step-0
rates in Table~\ref{tab:main} gives caps of $20,15,17,30$ for the four GPT-5.5
runs and $10,10,14,14$ for the four Claude-4.6 runs, respectively, rather than
the distribution-free cap $\lfloor1/\rho\rfloor=50$. These values are
illustrative because the held-out rates are not the panel scores used by the
gate; the exact certificate uses the recorded $\widehat F_0$ for each run.

Pooling $24{,}600$ task-level candidate-minus-current differences gives mean
$-0.006184$ and variance $0.024954$, rounded to $\sigma^2=0.025$; the baseline
unstable fraction averages $h=0.2075$, giving $\mathrm{vmass}=0.069$ from
\eqref{eq:vmass}. Two further variance bounds enter the certificates:
$\sigma_{\mathcal{H}}^2$, taken equal to $\sigma^2$, and $\sigma_W^2$, taken as
$0.05$; the certificates are stated as holding under these bounds.

At these values the radii are $4.8$ points for the final gain, $6.5$ points for
the drawdown at $A=8$, and $3.0$ points for the four-repository separation, so
the measured $0.92$, $17.2$ and $8.0$ points certify $\mathrm{DD}\le7.5$,
a gain $\ge+12.4$ and a separation $\ge+5.0$ points respectively. Inverting the
radius, a $10$-point drawdown tolerance requires $51$ evaluation tasks at $A=8$,
and the same tolerance on every accepted update of a run making at most seven of
them requires $136$ panel tasks at $T=60$.

\section{Implementation details}
\label{app:impl}
\subsection{Data construction and reference configuration}
\label{app:impl-reference}

We use one common Extract--Retrieve--Validate protocol across repository-level
experiments; only the repository, backbone, or explicitly ablated component
changes.  Table~\ref{tab:impl-config} summarizes this reference protocol rather
than a configuration specialized to any single repository.

\begin{table}[h]
\centering
\small
\begin{tabular}{@{}p{0.27\linewidth}p{0.65\linewidth}@{}}
\toprule
Component & Reference setting \\
\midrule
Data & One repository-level task stream from the SWE-smith training split \\
Evaluation set & 100 tasks sampled without replacement (seed 42) \\
Validation panel & 50 tasks sampled from the remaining stream (seed 123); the
same fixed panel is used for every acceptance decision \\
Agent & Frozen code-agent backbone (MiniMax-M2.7, GPT-5.5, or
Claude-4.6-sonnet, according to the run) with an interactive shell environment \\
Adaptation & 3 independently sampled rollouts per task \\
Evolution trigger & Mixed-outcome tasks only; at most 100 mixed tasks \\
Persistent state & Separate rich-text principle and skill banks \\
Retrieval & Task-conditioned top-10 principles and top-10 skills \\
Validation & Paired candidate/current comparison on 50 tasks, 3 rollouts per
condition and task, acceptance threshold $\rho=0.02$ \\
Evaluation & 100 held-out tasks at step 0 and every 10 mixed tasks, with 3
rollouts per task \\
\bottomrule
\end{tabular}
\caption{Reference configuration shared by the main repository-level runs.}
\label{tab:impl-config}
\end{table}

\paragraph{Repository selection.}
We study DeepDiff, Pygments, pydicom, and SQLFluff, selected from the
high-volume SWE-smith repositories using two practical criteria.  First, each
provides more than $1{,}000$ nonempty tasks to construct substantial adaptation,
validation, and evaluation partitions, and yields reliable patch-producing
agent runs. Second, they exercise markedly different forms of repository
knowledge: recursive object comparison, syntax highlighting, medical-image
processing, and SQL parsing and linting.  This combination lets us test whether
persistent skills help across distinct APIs, test structures, and debugging
patterns rather than within variants of one application domain.

Within each repository, the partition is constructed before adaptation.  We
first sample and remove the evaluation set, then reserve 50 validation tasks,
and use the remainder as the adaptation stream. Thus no exact instance appears
in more than one partition. The 50 reserved tasks form one fixed validation
panel reused at every comparison rather than a fresh sample at each update. The
prepared order of the adaptation stream is preserved, and independent fixed
seeds make the partitions reproducible.

\subsection{Rollout collection and evolution trigger}
\label{app:impl-agent}

For each adaptation task $x_t$, the agent receives the issue description and an
interactive repository environment in which it can inspect files, edit source
code, and execute tests.  An attempt terminates when the agent submits a patch
or reaches its interaction budget.  The submitted patch is evaluated by the
executable task-specific test harness; this outcome, rather than the agent's
self-assessment, supplies the binary label.

Before generating attempts, Retrieve selects one memory context $M_t$ from the
current bank.  Three rollouts are then sampled independently under the same
issue and frozen context.  Let $R=3$ and let $z_{t,r}\in\{0,1\}$ indicate
whether attempt $r$ passes the executable tests.  We call task $t$ \emph{mixed}
when
\begin{equation}
  0 < \sum_{r=1}^{R}z_{t,r} < R.
  \label{eq:impl-mixed}
\end{equation}
Mixed tasks provide a within-task contrast: the issue, repository state, and
retrieved context are held fixed, but the sampled trajectories contain both
successes and failures.  The full pipeline applies Extract only to this subset,
avoiding lessons inferred solely from uniformly successful or uniformly failed
experience.  All tasks are nevertheless solved and scored.  The evolution
index counts mixed tasks, so the 100-step budget can require observing
approximately $1{,}000$ raw tasks.

\subsection{Extract: contrastive reflection}
\label{app:impl-extract}

Extract maps the labeled rollout set
$\tau_t=\{(h_{t,r},p_{t,r},z_{t,r})\}_{r=1}^{R}$ to a candidate update
$\Delta_t=(\Delta_t^{P},\Delta_t^{S})$, where $h$ is the interaction history
and $p$ the submitted patch.  The reflection context
contains the issue, the three labeled summaries, their final results, and the
names of existing memories.  The latter discourages paraphrased duplicates.
Extract returns at most three principles and three skills.

We distinguish the two memory types semantically.  A \emph{principle} captures
a domain-agnostic reasoning strategy---how to reason under uncertainty.  A
\emph{skill} captures a narrow, executable software-engineering procedure,
search strategy, testing pattern, or recurring mistake to avoid.  The central
system instruction used for same-task contrastive extraction is shown below;
formatting examples are omitted for space.

\begin{center}
\setlength{\fboxsep}{8pt}
\fcolorbox{circlednum}{locushl}{%
\begin{minipage}{0.91\linewidth}
\small
\textbf{System prompt: contrastive Extract}\par\medskip
\emph{You are a learning agent extracting transferable lessons from multiple
attempts at the same software engineering task. Return strict JSON with exactly
three keys: issue\_summary, principles, skills.}\par\medskip
\emph{Focus on: what reasoning distinguished successful attempts from failed
ones; mistakes repeated across failed attempts; concrete techniques that
worked.}\par\medskip
\textbf{Principles:} \emph{meta-cognitive, domain-agnostic reasoning
strategies. Principles describe HOW to think and reason, not what to do
procedurally. They must be useful beyond any single language, framework, or
task type.}\par\medskip
\textbf{Skills:} \emph{specific, actionable software-engineering techniques.
They must be narrow enough to apply immediately and grounded in what went wrong
or what worked in the trajectories.}\par\medskip
\emph{Output only valid JSON. Return at most three new principles and three new
skills; return empty lists if there are no new insights worth adding.}
\end{minipage}}
\end{center}

The accompanying user message instantiates the following template:
\begin{center}
\setlength{\fboxsep}{8pt}
\fcolorbox{black!45}{black!3}{%
\begin{minipage}{0.91\linewidth}
\small
\textbf{User prompt template}\par\medskip
\emph{Task: [issue description]}\par
\emph{This task was attempted $R$ times: [number] succeeded and [number]
failed.}\par\medskip
\emph{Attempt 1 [SUCCESS/FAILED]: [trajectory summary; final result]}\par
\emph{$\ldots$}\par
\emph{Attempt $R$ [SUCCESS/FAILED]: [trajectory summary; final result]}\par\medskip
\emph{Already known principles/skills (do not repeat or closely paraphrase):
[item names]. Extract at most three new principles and three new skills.}
\end{minipage}}
\end{center}

Each principle contains a short search sentence and a structured body with
\textbf{Statement} and \textbf{When to apply}.  Each skill contains a search
sentence, a scope hint, and a structured body with \textbf{Scope},
\textbf{Trigger}, and a numbered \textbf{Procedure}.  Optional anti-patterns
and code examples make the learned text directly usable by a future solver.
The logical output schema is
\begin{verbatim}
{
  "principles": [{"name": "...", "text": "...",
                  "markdown": "..."}],
  "skills": [{"name": "...", "text": "...",
              "scope_hint": "...", "markdown": "..."}]
}
\end{verbatim}
Malformed outputs produce an empty proposal rather than an unconstrained repair
attempt.  This conservative choice prevents malformed reflections from entering
the persistent state.

\subsection{Persistent skill state}
\label{app:impl-memory}

In the reference configuration, the state $M=(P,S)$ consists of separate
append-only banks of principles $P$ and skills $S$. An item stores its name,
search sentence, rich body, source
tasks, support count, creation step, usage count, and source-outcome statistics;
skills additionally retain a scope hint.  These fields support both retrieval
and later inspection of provenance.

Before validation, exact normalized duplicates are merged with their existing
items.  A candidate is also suppressed when its token-set Jaccard similarity to
an existing item is at least $0.45$:
\begin{equation}
 J(a,b)=\frac{|T_a\cap T_b|}{|T_a\cup T_b|}\geq 0.45.
 \label{eq:impl-dedup}
\end{equation}
Here $T$ includes the search sentence and rich body, as well as the scope hint
for skills.  The reference configuration performs no periodic summarization or
LLM-based consolidation, ensuring that accepted updates remain individually
traceable.

\subsection{Retrieve: task-conditioned skill composition}
\label{app:impl-retrieve}

Retrieve independently ranks principles and skills for the current issue.  The
query combines the task identifier and issue description.  For item $i$, let
$Q$ denote query tokens, $C_i$ the item's search-text tokens, and $M_i$ its
metadata tokens (the scope hint for a skill and the empty set for a principle).
The score is
\begin{align}
 s(i,q) ={}& \underbrace{\frac{|Q\cap(C_i\cup M_i)|}
  {\min(|Q|,|C_i\cup M_i|)}}_{\text{overlap coefficient}}
 + \underbrace{\sum_{w\in Q\cap C_i}\log_2\!\max\!\left(\frac{N}{df(w)},1\right)}
 _{\text{text IDF}} \nonumber\\
 &+0.45\underbrace{\sum_{w\in Q\cap M_i}\log_2\!\max\!\left(\frac{N}{df(w)},1\right)}
 _{\text{scope IDF}}
 +0.5I_i+\cos(e_q,e_i),
 \label{eq:impl-retrieval}
\end{align}
where $N$ is the size of the corresponding bank.  The lexical terms reward
direct issue and scope matches, while cosine similarity captures semantic
relatedness between the issue and rich item body. $I_i$ is an importance prior
derived from source-task outcomes, with support count as a fallback. Semantic vectors
use the Qwen3-Embedding-0.6B model when available, with a deterministic hashing
representation as a fallback.  The ten highest-scoring principles and ten
highest-scoring skills are selected independently.

The selected rich bodies are composed into the solver context rather than
re-summarized.  This preserves procedures, triggers, and anti-patterns exactly
as validated.  The injected system-prompt block has the following form:
\begin{center}
\setlength{\fboxsep}{8pt}
\fcolorbox{circlednum}{locushl}{%
\begin{minipage}{0.91\linewidth}
\small
\textbf{System-prompt augmentation produced by Retrieve}\par\medskip
\emph{Learned principles:}\par
\emph{--- [principle name] ---}\par
\emph{[Statement; when to apply; optional anti-pattern]}\par
\emph{$\ldots$}\par\medskip
\emph{Learned skills:}\par
\emph{--- [skill name] ---}\par
\emph{[Scope; trigger; numbered procedure; optional anti-pattern]}\par
\emph{$\ldots$}\par\medskip
\emph{Use the above when relevant, but do not follow them blindly if repository
evidence contradicts them.}
\end{minipage}}
\end{center}
The final instruction treats retrieved memory as defeasible guidance rather than
an authority that can override repository evidence.

\subsection{Validate: paired acceptance on executable outcomes}
\label{app:impl-validate}

Extract proposes a \emph{joint} bank move containing all new principles and
skills from the current mixed task.  Validate compares two conditions on the
same fixed 50-task validation set:
\begin{itemize}
  \item \textbf{current}: task-conditioned retrieval from the current bank $M$;
  \item \textbf{candidate}: the same retrieved context augmented by all items
  in $\Delta_t$.
\end{itemize}
Each condition receives three independently sampled and executably graded
rollouts per task.  For $c\in\{\mathrm{cur},\mathrm{cand}\}$, the empirical pass
rate is
\begin{equation}
 \widehat p_c=\frac{1}{50}\sum_{j=1}^{50}
   \frac{1}{3}\sum_{r=1}^{3}z^{(c)}_{j,r}.
 \label{eq:impl-validation-rate}
\end{equation}
All reported experiments use the acceptance threshold $\rho=0.02$. The proposed
move is committed exactly when
\begin{equation}
 \widehat p_{\mathrm{cand}}\geq \widehat p_{\mathrm{cur}}+0.02.
 \label{eq:impl-acceptance}
\end{equation}
Thus each condition aggregates 150 binary outcomes, and the threshold requires
at least three net additional successful rollouts on the $1/150$ score grid. The
comparison validates the joint state transition rather than attributing utility
to individual items.
Results for the unchanged current state are reused until an accepted update
changes the bank, reducing repeated evaluation without changing the acceptance
rule.

The same 50 tasks are reused for every candidate. They are 50 distinct task
draws with three within-task replications, not 150 independent population-task
samples. Accordingly, $\rho=0.02$ is an empirical acceptance margin rather than
a confidence level. A comparison uses
up to $2\times50\times3=300$ executable rollouts before reuse of the unchanged
current result. Across $24{,}600$ task-level paired samples from the main runs,
the candidate-minus-current difference has mean $-0.006184$ and variance
$0.024954$, rounded to $\sigma^2=0.025$. Theorem~\ref{thm:conv-main}
characterizes fixed-panel convergence with score carry-forward, while
Theorem~\ref{thm:reuse-main} gives the corresponding uniform population-loss
guarantee over the paths induced by panel reuse. Table~\ref{tab:sizing} and
Figure~\ref{fig:sizing} report the resulting sizing rules; the numerical
instantiation is detailed in Appendix~\ref{app:holdout-choice}.

\subsection{Evaluation protocol and reported statistic}
\label{app:impl-evaluation}

Before adaptation and every ten mixed tasks, we freeze the current bank and
evaluate it on a disjoint 100-task panel.  Each task receives three independent
rollouts.  For panel size $H=100$, avg@3 is
\begin{equation}
 \operatorname{avg@3}=100\times
 \frac{\sum_{j=1}^{H}\sum_{r=1}^{3}z_{j,r}}{3H}.
 \label{eq:impl-avg3}
\end{equation}
The denominator includes all $3H$ attempts, including empty or unsuccessful
submissions. The main table reports the final-checkpoint avg@3 first and the
highest observed avg@3 over the trajectory in parentheses. The latter uses this
panel for checkpoint selection and therefore is not an untouched final-test
estimate. When gains are discussed, final and peak gains are computed relative
to the shared step-0 evaluation. We retain the bank state, extracted proposals, rollout
histories, executable verdicts, validation comparisons, and periodic evaluation
results needed to reconstruct each accepted transition.

\subsection{Ablation implementations and prompts}
\label{app:impl-ablations}

Table~\ref{tab:impl-ablations} describes each intervention at the method level.
Because some interventions jointly alter representation and control flow, we
state these couplings explicitly rather than treating every command-line change
as a perfectly isolated causal operator.

\begin{table}[h]
\centering
\footnotesize
\begin{tabular}{@{}p{0.19\linewidth}p{0.73\linewidth}@{}}
\toprule
Ablation & Method-level intervention \\
\midrule
Sentence skills & Replace structured rich bodies with one-sentence principles
and scoped skill sentences throughout extraction, storage, and injection. \\
No principles & Constrain Extract to return an empty principle set; retrieve and
inject skills only. \\
Single document & Replace itemized banks with one principle document and one
skill document, updated through section-level add/replace/delete operations and
injected in full.  This also removes item-level retrieval. \\
Inject all & Remove task-conditioned ranking and inject every stored item in
reverse creation order. \\
Batch of four & Buffer four completed tasks and perform one cross-task
reflection; use one validation rollout per validation task. \\
All tasks & Apply Extract and Validate after all-pass and all-fail tasks as well
as mixed tasks; the 100-step horizon consequently counts raw tasks. \\
\bottomrule
\end{tabular}
\caption{Method-level definitions of the ablations.}
\label{tab:impl-ablations}
\end{table}

\paragraph{Batch extraction.}
For the batch-size ablation, Extract receives four task blocks, each containing
the issue, outcome label, trajectory summary, and final result.  Principles
must recur across at least two tasks or be sufficiently universal to justify
retention from one instance; skills may remain task-specific when reusable.
The core system instruction is:
\begin{center}
\setlength{\fboxsep}{8pt}
\fcolorbox{black!45}{black!3}{%
\begin{minipage}{0.91\linewidth}
\small
\textbf{System prompt: batched Extract}\par\medskip
\emph{You are a learning agent that extracts transferable lessons by comparing
multiple completed software engineering tasks.}\par\medskip
\emph{Principles must appear as a recurring pattern across at least two tasks
in this batch, OR be so universally applicable that a single instance justifies
inclusion. Skills can be task-specific as long as they are genuinely reusable.}
\end{minipage}}
\end{center}
In mixed-only mode, a batch is reflected upon when it contains at least one
mixed task, but neighboring all-pass or all-fail examples remain visible to the
cross-task reflection.  Since this condition also uses one validation rollout
per task, it changes both extraction granularity and validation noise.

\paragraph{Single-document extraction.}
The single-document condition represents each memory type as sections under
named headings.  Reflection proposes bounded section edits of the form
\begin{center}
\setlength{\fboxsep}{8pt}
\fcolorbox{black!45}{black!3}{%
\begin{minipage}{0.91\linewidth}
\small
\textbf{Single-document update schema}\par\medskip
\emph{operation: add, replace, or delete}\par
\emph{target: existing section name}\par
\emph{name: proposed section name}\par
\emph{text: one-sentence summary}\par
\emph{markdown: complete replacement body}\par
\emph{scope hint: where or when a skill applies}
\end{minipage}}
\end{center}
A replacement removes the target section and appends the proposed version; a
missing replacement target becomes an addition, while an addition that reuses
an existing name is ignored.  Because the resulting documents are always
injected in full, this condition jointly changes extraction, storage, and
retrieval rather than merely changing the on-disk format.

\paragraph{Interpretive scope.}
The remaining couplings are similarly important.  All-task evolution changes
both the extraction trigger and the number of raw tasks observed before
stopping; removing validation also removes validation-derived reflection
feedback; and high reasoning changes reflection as well as solving.  Ablation
differences should therefore be interpreted as effects of these operational
method variants, not as estimates of perfectly isolated primitives.

\section{Additional ablation and skill-bank details}
\label{app:bank}

\subsection{Prior work behind each ablated component}
\label{app:ablation-design}

Table~\ref{tab:ablation-design} maps both sides of every component ablation in
Figure~\ref{fig:ablation} to existing self-evolving methods.

\begin{table}[h]
\centering
\footnotesize
\setlength{\tabcolsep}{4pt}
\renewcommand{\arraystretch}{1.2}
\begin{tabular}{@{}>{\raggedright\arraybackslash}p{0.21\linewidth}
                >{\raggedright\arraybackslash}p{0.37\linewidth}
                >{\raggedright\arraybackslash}p{0.37\linewidth}@{}}
\toprule
\textbf{Component} & \textbf{Prior work with our choice}
& \textbf{Prior work with the ablated choice} \\
\midrule
\textbf{Retrieval}\newline{\color{black!55}Selection $\to$ inject all}
  & Voyager~\citep{wang2023voyager}, ReasoningBank~\citep{ouyang2025reasoningbank}, Memp~\citep{fang2025memp}
  & AWM~\citep{wang2024awm}, ExpeL~\citep{zhao2024expel}, SkillOpt~\citep{yang2026skillopt} \\
\textbf{Training stream}\newline{\color{black!55}Contrastive only $\to$ all tasks}
  & ExpeL~\citep{zhao2024expel}, H$^2$R~\citep{ye2025h2r}, Training-Free GRPO~\citep{cai2025trainingfreegrpo}
  & ReasoningBank~\citep{ouyang2025reasoningbank}, ACE~\citep{zhang2025ace}, SkillOpt~\citep{yang2026skillopt} \\
\textbf{Principle memory}\newline{\color{black!55}Remove principles}
  & MARS~\citep{hou2026mars}, SkillRL~\citep{xia2026skillrl}, H$^2$R~\citep{ye2025h2r}
  & Voyager~\citep{wang2023voyager}, AWM~\citep{wang2024awm} \\
\textbf{Skill format}\newline{\color{black!55}Rich markdown $\to$ sentence}
  & Trace2Skill~\citep{ni2026trace2skill}, EvoSkill~\citep{alzubi2026evoskill}, SkillOpt~\citep{yang2026skillopt}
  & ExpeL~\citep{zhao2024expel}, ACE~\citep{zhang2025ace}, Training-Free GRPO~\citep{cai2025trainingfreegrpo} \\
\textbf{Skill storage}\newline{\color{black!55}Bank $\to$ single document}
  & SkillRL~\citep{xia2026skillrl}, ReasoningBank~\citep{ouyang2025reasoningbank}, SWE-Exp~\citep{chen2025sweexp}
  & SkillOpt~\citep{yang2026skillopt}, ACE~\citep{zhang2025ace}, Dynamic Cheatsheet~\citep{suzgun2025dynamiccheatsheet} \\
\textbf{Generation granularity}\newline{\color{black!55}Batch size $1 \to 4$}
  & ReasoningBank~\citep{ouyang2025reasoningbank}, ACE~\citep{zhang2025ace}, Dynamic Cheatsheet~\citep{suzgun2025dynamiccheatsheet}
  & SkillOpt~\citep{yang2026skillopt}, Trace2Skill~\citep{ni2026trace2skill}, Training-Free GRPO~\citep{cai2025trainingfreegrpo} \\
\bottomrule
\end{tabular}
\caption{Prior self-evolving methods that adopt our choice or the ablated
choice for each component in Figure~\ref{fig:ablation}.}
\label{tab:ablation-design}
\end{table}

\subsection{Growth and composition of the skill bank}

Figure~\ref{fig:bank-growth-example} in the main text illustrates the evolution
dynamics for Claude-4.6 on \texttt{pygments}. With validation, the bank grows through seven accepted
updates, reaches $26$ items ($12$ skills and $14$ principles) at step $323$,
and then remains fixed. Without validation, proposed items accumulate until the
bank size reaches $453$ items ($249$ skills and $204$ principles) by
approximately step $400$. Despite this much larger bank, the ungated run falls
from its early peak to an $80.0\%$ final resolved rate, whereas the compact
validated bank retains $92.3\%$.

Table~\ref{tab:bank-size} in the main text extends this comparison to all main runs. Validation
retains $204$ skills and $192$ principles in total, compared with $2{,}422$
skills and $2{,}024$ principles without validation. The gate therefore produces
banks that are $11\times$ more compact rather than allowing extracted items to
accumulate unconditionally.

\section{Time and efficiency analysis}
\label{app:efficiency}

We report total wall-clock time for each main model--repository run. For the
component ablations, whose evolution horizons differ by design, we retain the
per-step normalization used in the original runtime logs.

\subsection{Main experiment runtime}

Table~\ref{tab:main-time} reports total runtime for the main SWE-smith
experiments. Averaged across the twelve model--repository runs, ungated
evolution takes $50.6$ hours, validated evolution takes $122.1$ hours, and
SkillOpt takes $96.1$ hours.

\begin{table}[H]
\centering
\footnotesize
\setlength{\tabcolsep}{3.5pt}
\renewcommand{\arraystretch}{1.15}
\begin{tabular}{@{}llccccc@{}}
\toprule
& & \multicolumn{5}{c}{\textbf{Train+validation wall-clock time (h)}} \\
\cmidrule(lr){3-7}
\textbf{Model} & \textbf{Setting} & \textbf{\texttt{pygments}}
& \textbf{\texttt{sqlfluff}} & \textbf{\texttt{deepdiff}}
& \textbf{\texttt{pydicom}} & \textbf{Model Avg.} \\
\midrule
 \textbf{MiniMax-M2.7} & SkillOpt
  & $118.5$ & $134.1$ & $105.7$ & $121.7$ & $120.0$ \\
  & Evolve w/o val
  & $42.9$ & $46.3$ & $47.2$ & $73.9$ & $52.6$ \\
  & \cellcolor{locushl!55}Evolve w/ val
  & \cellcolor{locushl!55}$127.5$ & \cellcolor{locushl!55}$107.8$ & \cellcolor{locushl!55}$139.0$ & \cellcolor{locushl!55}$153.5$ & \cellcolor{locushl!55}$131.9$ \\
\addlinespace[3pt]
 \textbf{GPT-5.5} & SkillOpt
  & $131.9$ & $140.5$ & $90.5$ & $132.6$ & $123.9$ \\
  & Evolve w/o val
  & $40.1$ & $97.6$ & $55.5$ & $67.3$ & $65.1$ \\
  & \cellcolor{locushl!55}Evolve w/ val
  & \cellcolor{locushl!55}$103.9$ & \cellcolor{locushl!55}$128.2$ & \cellcolor{locushl!55}$101.5$ & \cellcolor{locushl!55}$148.8$ & \cellcolor{locushl!55}$120.6$ \\
\addlinespace[3pt]
 \textbf{Claude-4.6} & SkillOpt
  & $51.8$ & $30.5$ & $29.2$ & $66.5$ & $44.5$ \\
  & Evolve w/o val
  & $32.8$ & $57.6$ & $22.5$ & $23.5$ & $34.1$ \\
  & \cellcolor{locushl!55}Evolve w/ val
  & \cellcolor{locushl!55}$126.2$ & \cellcolor{locushl!55}$119.3$ & \cellcolor{locushl!55}$54.7$ & \cellcolor{locushl!55}$154.2$ & \cellcolor{locushl!55}$113.6$ \\
\midrule
 \textbf{Repo average} & SkillOpt
  & $100.7$ & $101.7$ & $75.1$ & $107.0$ & $96.1$ \\
  & Evolve w/o val
  & $38.6$ & $67.2$ & $41.7$ & $54.9$ & $50.6$ \\
  & \cellcolor{locushl!55}Evolve w/ val
  & \cellcolor{locushl!55}$119.2$ & \cellcolor{locushl!55}$118.4$ & \cellcolor{locushl!55}$98.4$ & \cellcolor{locushl!55}$152.2$ & \cellcolor{locushl!55}$122.1$ \\
\bottomrule
\end{tabular}
\caption{Total wall-clock runtime in hours for the main SWE-smith experiments,
with held-out test evaluation excluded. The \emph{Model Avg.} column averages
repositories within each model, and the \emph{Repo average} block averages
models within each repository. MiniMax-M2.7 is served on local GPUs, so its
times also reflect local inference throughput.}
\label{tab:main-time}
\end{table}

\subsection{Ablation runtime}

Table~\ref{tab:ablation-time} reports runtime per evolution step for the GPT-5.5
\texttt{pydicom} ablations. The variants differ substantially in their work per
step: evolving on all tasks is the fastest at $1.14$ hours per step, while the
sentence-format variant is the slowest among the reported component ablations
at $2.41$ hours per step.

\begin{table}[t]
\centering
\footnotesize
\setlength{\tabcolsep}{6pt}
\renewcommand{\arraystretch}{1.15}
\begin{tabular}{@{}lll r@{}}
\toprule
\textbf{ID} & \textbf{Component} & \textbf{Configuration}
& \textbf{Hours / step} \\
\midrule
\rowcolor{locushl!55}
Ref. & Full pipeline & Mixed-only, validate $50$ & $2.11$ \\
A1 & Skill format & Sentence & $2.41$ \\
A2 & Principle memory & No principles & $2.20$ \\
A3 & Storage & Single document & $2.08$ \\
A4 & Selection & Inject all & $1.83$ \\
A5 & Extraction batch & Batch size $4$ & $1.58$ \\
A6 & Training stream & All tasks & $1.14$ \\
\bottomrule
\end{tabular}
\caption{Wall-clock runtime per evolution step for the GPT-5.5
\texttt{pydicom} component ablations. The reference configuration is shaded.
Only per-step time is reported; total runtime and the number of evolved steps
are omitted.}
\label{tab:ablation-time}
\end{table}

\section{Case study of useful generated skills}
\label{app:cases}

We analyze the final bank from the Claude-4.6 evolution run on
\texttt{pygments}. Each learned skill or principle is evaluated individually
on the same $100$-task holdout with three rollouts per task. The reported gain
is the change in resolved rate from injecting that item alone relative to the
no-memory baseline. ``Steps'' is the number of adaptation tasks processed when
the item entered the bank; equal values indicate items created by the same
accepted update.

\begingroup
\small
\setlength{\tabcolsep}{3pt}
\renewcommand{\arraystretch}{1.08}
\begin{longtable}{@{}r p{0.48\linewidth} p{0.10\linewidth} p{0.16\linewidth} r@{}}
\caption{Single-item utility for all items in the final Claude-4.6
\texttt{pygments} bank, listed in learning order. Gain is measured relative to
the no-memory baseline on the $100$-task holdout.}
\label{tab:artifact-utility}\\
\toprule
\textbf{Steps} & \textbf{Item} & \textbf{Type} & \textbf{Category} & \textbf{Gain} \\
\midrule
\endfirsthead
\multicolumn{5}{c}{\tablename\ \thetable\ continued}\\
\toprule
\textbf{Steps} & \textbf{Item} & \textbf{Type} & \textbf{Category} & \textbf{Gain} \\
\midrule
\endhead
\midrule
\multicolumn{5}{r}{Continued on next page}\\
\endfoot
\bottomrule
\endlastfoot
1 & \nolinkurl{scope_reported_vs_scope_tested} & Principle & Repair strategy & $+5.33$ \\
1 & \nolinkurl{diff_from_baseline_before_fixing} & Principle & Repair strategy & $+9.00$ \\
1 & \nolinkurl{audit_full_modified_file_for_multiple_bugs} & Skill & Repair strategy & $+2.00$ \\
2 & \nolinkurl{partial_fix_can_fail_tests_that_full_fix_passes} & Principle & Repair strategy & $+5.00$ \\
2 & \nolinkurl{co_introduced_bugs_cluster_in_same_commit} & Principle & Repair strategy & $+7.00$ \\
2 & \nolinkurl{verify_fix_completeness_with_full_test_suite} & Skill & Repair strategy & $+2.33$ \\
2 & \nolinkurl{git_show_commit_enumerate_all_changes} & Skill & Repair strategy & $+8.00$ \\
214 & \nolinkurl{operator_substitution_changes_return_semantics} & Principle & Code semantics & $+2.00$ \\
214 & \nolinkurl{boolean_inversion_has_global_side_effects} & Principle & Code semantics & $+1.67$ \\
214 & \nolinkurl{test_analyse_text_with_non_matching_input} & Skill & Code semantics & $+1.33$ \\
217 & \nolinkurl{deferred_import_requires_empty_initialization} & Principle & Code semantics & $+2.00$ \\
217 & \nolinkurl{identical_patch_different_test_outcome_signals_environment_issue} & Principle & Environment & $+1.33$ \\
217 & \nolinkurl{audit_set_operations_for_semantic_correctness} & Skill & Code semantics & $+1.33$ \\
217 & \nolinkurl{check_kwargs_forwarding_in_super_init} & Skill & Code semantics & $+0.33$ \\
241 & \nolinkurl{inverted_condition_implies_inverted_action} & Principle & Code semantics & $+1.33$ \\
241 & \nolinkurl{failed_attempt_same_patch_signals_test_environment} & Principle & Environment & $+1.00$ \\
241 & \nolinkurl{install_test_dependencies_before_running_tests} & Skill & Environment & $+1.67$ \\
241 & \nolinkurl{fix_return_value_alongside_condition_fix} & Skill & Code semantics & $+1.33$ \\
262 & \nolinkurl{identical_patch_nondeterministic_outcome_implies_external_factor} & Principle & Environment & $+1.67$ \\
262 & \nolinkurl{revert_to_known_good_as_fix_strategy} & Principle & Repair strategy & $+9.33$ \\
262 & \nolinkurl{enumerate_all_swapped_values_in_buggy_commit} & Skill & Repair strategy & $+10.33$ \\
262 & \nolinkurl{verify_patch_against_known_good_commit} & Skill & Repair strategy & $+5.33$ \\
323 & \nolinkurl{variable_use_before_assignment_signals_reordered_code} & Principle & Code semantics & $+3.33$ \\
323 & \nolinkurl{none_check_polarity_determines_fallback_path} & Principle & Code semantics & $+0.33$ \\
323 & \nolinkurl{audit_yield_tuple_index_group_consistency} & Skill & Code semantics & $+2.00$ \\
323 & \nolinkurl{trace_all_four_bug_classes_in_single_commit} & Skill & Repair strategy & $+7.00$ \\
\end{longtable}
\endgroup

The learned skills and principles capture knowledge at different levels of
abstraction. Skills generally encode specific, executable repair procedures.
For example, \nolinkurl{audit_yield_tuple_index_group_consistency} records a
Pygments-specific invariant: a lexer callback must take the token position and
value from the same regex capture group. The skill
\nolinkurl{enumerate_all_swapped_values_in_buggy_commit} captures a broader
workflow for multi-edit regressions: inspect the complete commit, classify
every mutation, construct a checklist, and restore every incorrect change.

Principles instead abstract a general diagnostic lesson from a particular
failure. \nolinkurl{variable_use_before_assignment_signals_reordered_code}
recommends restoring the intended data-flow order rather than masking the
symptom with an arbitrary default value. Together, these examples show that
evolution acquires both repository-specific operational knowledge and
transferable coding principles. We reproduce two detailed skills and one
general principle below using the text from the final bank.

\phantomsection\label{app:item-audit-yield}
\begin{center}
\begin{minipage}{\linewidth}
\centering
\setlength{\fboxsep}{0pt}
\setlength{\fboxrule}{0.5pt}
\fcolorbox{black!30}{white}{%
\begin{minipage}{0.96\linewidth}
{\setlength{\fboxsep}{6pt}%
\colorbox{black!7}{\parbox{\dimexpr\linewidth-12pt\relax}{%
\strut\scriptsize\textcolor{black!55}{\ding{110}\enspace}\enspace
\texttt{audit\_yield\_tuple\_index\_group\_consistency.md}%
\hfill\textcolor{black!40}{$\times$}}}}\par
\vspace{9pt}
\hspace*{10pt}\begin{minipage}{\dimexpr\linewidth-20pt\relax}
\small
{\large\bfseries audit\_yield\_tuple\_index\_group\_consistency}\par
\vspace{3pt}{\color{black!20}\hrule height 0.5pt}\vspace{7pt}
\textbf{Scope:} Any lexer callback method that yields
\texttt{(match.start(N), TokenType, match.group(N))} tuples.\par\medskip
\textbf{Trigger:} A bug report or diff shows that match group indices were
swapped between \texttt{match.start()} and \texttt{match.group()} calls in
yield statements.\par\medskip
\textbf{Procedure:}\par
\begingroup
\renewcommand{\labelenumi}{(\arabic{enumi})}
\begin{enumerate}
\setlength{\topsep}{3pt}
\setlength{\partopsep}{0pt}
\setlength{\itemsep}{1pt}
\setlength{\parsep}{0pt}
\item Locate every \texttt{yield} statement in the callback that uses
\texttt{match.start(N)} and \texttt{match.group(M)}.
\item For each yield, confirm that \texttt{N == M}---the position index and the
group index must refer to the same capture group.
\item If they differ (e.g.,
\texttt{yield match.start(2), String, match.group(1)}), swap them back so both
indices match.
\item Cross-check against the regex pattern's group numbering to confirm the
semantic meaning of each group.
\end{enumerate}
\endgroup
\medskip
\textbf{Anti-pattern:}\par\smallskip
{\setlength{\fboxsep}{7pt}%
\setlength{\fboxrule}{0.4pt}%
\fcolorbox{black!15}{black!3}{%
\begin{minipage}{\dimexpr\linewidth-14.8pt\relax}
\footnotesize\ttfamily
\textcolor{gaingreen}{\# Wrong: position from group 2, value from group 1}\par
\textcolor{blue!70!black}{yield match}.start(\textcolor{gaingreen}{2}), String,
\textcolor{blue!70!black}{match}.group(\textcolor{gaingreen}{1})\par
\textcolor{blue!70!black}{yield match}.start(\textcolor{gaingreen}{1}), String,
\textcolor{blue!70!black}{match}.group(\textcolor{gaingreen}{2})\par\medskip
\textcolor{gaingreen}{\# Correct: position and value from the same group}\par
\textcolor{blue!70!black}{yield match}.start(\textcolor{gaingreen}{1}), String,
\textcolor{blue!70!black}{match}.group(\textcolor{gaingreen}{1})\par
\textcolor{blue!70!black}{yield match}.start(\textcolor{gaingreen}{2}), String,
\textcolor{blue!70!black}{match}.group(\textcolor{gaingreen}{2})
\end{minipage}}}
\end{minipage}\vspace{10pt}
\end{minipage}}
\par\smallskip
\small\emph{Code Semantics $\cdot$ Skill $\cdot$ evolved steps 323
$\cdot$ gain $+2.00$ pp.}
\end{minipage}
\end{center}

\phantomsection\label{app:item-enumerate-swaps}
\begin{center}
\begin{minipage}{\linewidth}
\centering
\setlength{\fboxsep}{0pt}
\setlength{\fboxrule}{0.5pt}
\fcolorbox{black!30}{white}{%
\begin{minipage}{0.96\linewidth}
{\setlength{\fboxsep}{6pt}%
\colorbox{black!7}{\parbox{\dimexpr\linewidth-12pt\relax}{%
\strut\scriptsize\textcolor{black!55}{\ding{110}\enspace}\enspace
\texttt{enumerate\_all\_swapped\_values\_in\_buggy\_commit.md}%
\hfill\textcolor{black!40}{$\times$}}}}\par
\vspace{9pt}
\hspace*{10pt}\begin{minipage}{\dimexpr\linewidth-20pt\relax}
\small
{\large\bfseries enumerate\_all\_swapped\_values\_in\_buggy\_commit}\par
\vspace{3pt}{\color{black!20}\hrule height 0.5pt}\vspace{7pt}
\textbf{Scope:} Functions where a bug-introducing commit swapped or inverted
one or more values (booleans, string literals, constants, token types).\par\medskip
\textbf{Trigger:} You find one swapped value (e.g., \texttt{\char39} vs.
\texttt{\char34}) in a function and are about to fix only that one.\par\medskip
\textbf{Procedure:}\par
\begingroup
\renewcommand{\labelenumi}{(\arabic{enumi})}
\begin{enumerate}
\setlength{\topsep}{3pt}
\setlength{\partopsep}{0pt}
\setlength{\itemsep}{1pt}
\setlength{\parsep}{0pt}
\item Run \texttt{git show <bad-commit> -- <file>} and collect every
\texttt{-} (removed) and \texttt{+} (added) line in the affected function.
\item For each changed line, categorize the change: swap, inversion, constant
change, reorder, deletion.
\item List all changes as a checklist before writing any fix.
\item Fix each item on the checklist, not just the first one you noticed.
\item After fixing, re-run \texttt{git show <bad-commit>} and confirm every
\texttt{-} line is restored and every \texttt{+} line is removed.
\end{enumerate}
\endgroup
\medskip
\textbf{Anti-pattern:}\par\smallskip
{\setlength{\fboxsep}{7pt}%
\setlength{\fboxrule}{0.4pt}%
\fcolorbox{black!15}{black!3}{%
\begin{minipage}{\dimexpr\linewidth-14.8pt\relax}
\footnotesize\ttfamily\color{gaingreen}
\# Bug commit changed:\par
\# char = r\char39\char34\char39 if double else
r\char34\char39\char34   (correct)\par
\# to:\par
\# char = r\char34\char39\char34 if double else
r\char39\char34\char39   (wrong)\par
\# AND also changed quantifier '\{3,\}' to '\{2,\}'\par
\# AND also changed state\_name prefix 't' to 'd'\par
\# Fixing only the char swap leaves the other bugs intact.
\end{minipage}}}
\end{minipage}\vspace{10pt}
\end{minipage}}
\par\smallskip
\small\emph{Repair Strategy $\cdot$ Skill $\cdot$ evolved steps 262
$\cdot$ gain $+10.33$ pp.}
\end{minipage}
\end{center}

\phantomsection\label{app:item-use-before-assignment}
\begin{center}
\begin{minipage}{\linewidth}
\centering
\setlength{\fboxsep}{0pt}
\setlength{\fboxrule}{0.5pt}
\fcolorbox{black!30}{white}{%
\begin{minipage}{0.96\linewidth}
{\setlength{\fboxsep}{6pt}%
\colorbox{black!7}{\parbox{\dimexpr\linewidth-12pt\relax}{%
\strut\scriptsize\textcolor{black!55}{\ding{110}\enspace}\enspace
\texttt{variable\_use\_before\_assignment\_signals\_reordered\_code.md}%
\hfill\textcolor{black!40}{$\times$}}}}\par
\vspace{9pt}
\hspace*{10pt}\begin{minipage}{\dimexpr\linewidth-20pt\relax}
\small
{\large\bfseries Use before assignment can indicate reordered code}\par
\vspace{3pt}{\color{black!20}\hrule height 0.5pt}\vspace{7pt}
\textbf{Statement:} When a variable is used before it is assigned in a code
block, suspect that lines were reordered or moved during a buggy edit rather
than that the variable is simply missing.\par\medskip
\textbf{When to apply:}\par
\begin{itemize}
\setlength{\topsep}{3pt}
\setlength{\partopsep}{0pt}
\setlength{\itemsep}{1pt}
\setlength{\parsep}{0pt}
\item You encounter a \texttt{NameError} or a use-before-assignment pattern in
code that previously worked.
\item A bug-introducing commit touched the ordering of statements in a
function.
\item The variable is assigned later in the same function, suggesting it was
moved rather than deleted.
\end{itemize}
\medskip
\textbf{Better pattern:} Trace the full sequence of statements in the function
and check whether any assignment was displaced below a use site. Reordering the
statements back to their logical order (assign before use) is often the complete
fix.\par\medskip
\textbf{Anti-pattern:} Assuming the variable needs to be initialized at the top
of the function with a default value, when the real fix is restoring the
original statement order.
\end{minipage}\vspace{10pt}
\end{minipage}}
\par\smallskip
\small\emph{Code Semantics $\cdot$ Principle $\cdot$ evolved steps 323
$\cdot$ gain $+3.33$ pp.}
\end{minipage}
\end{center}

\finishdocument

\end{document}